\documentclass[11pt,letter,twoside]{article}

\usepackage{mathrsfs}
\usepackage{amsmath} 
\usepackage{mathtools}
\usepackage{bbm}
\usepackage{latexsym}
\usepackage{dsfont}
\usepackage{pifont}
\usepackage{color} 
\usepackage{algorithm2e}
\usepackage{tikz}
\usepackage{times}
\usepackage{verbatim}

\usepackage{hyperref}
\definecolor{bl}{RGB}{20,20,150}
\hypersetup{colorlinks=true, citecolor=bl, linkcolor=bl, urlcolor=bl}
\usepackage{breakurl}

\graphicspath{{./Figures/}}

\newcommand{\Acal}{\mathcal{A}}
\newcommand{\Bcal}{\mathcal{B}}

\newcommand{\Ecal}{\mathcal{E}}

\newcommand{\Mcal}{\mathcal{M}}
\newcommand{\Ncal}{\mathcal{N}}
\newcommand{\Pcal}{\mathcal{P}}
\newcommand{\Scal}{\mathcal{S}}
\newcommand{\Xcal}{\mathcal{X}}
\newcommand{\Ycal}{\mathcal{Y}}
\newcommand{\Zcal}{\mathcal{Z}}

\newcommand{\N}{\mathbb{N}}
\newcommand{\R}{\mathbb{R}}

\newcommand{\RBM}{\operatorname{RBM}}

\newcommand{\LTF}{\operatorname{LTF}}

\newcommand{\supp}{\operatorname{supp}}

\newcommand{\sgn}{\operatorname{sgn}}
\newcommand{\argmax}{\operatorname{argmax}}
\newcommand{\hs}{\operatorname{hs}}
\newcommand{\Del}[2]{\Delta_{#1,#2}}

\newcommand{\unitlayer}[3]{ 
\framebox{\tikzset{node distance=.4cm, auto}
\begin{tikzpicture}[scale=0.9, every node/.style={transform shape}]
\tikzstyle{neuron}=[circle, draw=black, inner sep=.01cm, minimum size = .55cm]

\foreach \name / \i in {1,...,#2}
 \node[neuron] (I-\name) at (\i,0) {$#1_{\i}$};
 
 \node (I-dots) [node distance = .8cm, right of = I-#2] {$\cdots$};
 \node[neuron] (I-end)  [node distance = .8cm, right of = I-dots] {${#1}_{#3}$};
 
\end{tikzpicture}
}}

\renewenvironment{abstract}
{\centerline{\large\bf Abstract}\vspace{0.7ex}%
  \bgroup\leftskip 20pt\rightskip 20pt\small\noindent}%
{\par\egroup\vskip 0.25ex}

\newenvironment{keywords}
{\bgroup\leftskip 20pt\rightskip 20pt \small\noindent{\bf Keywords:} }%
{\par\egroup\vskip 0.25ex}

\usepackage{times}
\usepackage{amssymb}
\usepackage{amsthm}
\usepackage{authblk}
\usepackage{graphicx} 

\usepackage{natbib}
\bibpunct[; ]{(}{)}{;}{a}{}{;}

\definecolor{bl}{RGB}{20,20,150}
\usepackage{hyperref}
\hypersetup{colorlinks=true, citecolor=bl, linkcolor=bl, urlcolor=bl}

\newtheorem{theorem}{Theorem}
\newtheorem{lemma}[theorem]{Lemma}
\newtheorem{corollary}[theorem]{Corollary}
\newtheorem{proposition}[theorem]{Proposition}

\theoremstyle{definition}

\newtheorem{definition}[theorem]{Definition}
\newtheorem{remark}[theorem]{Remark}

\renewenvironment{proof}[1][\unskip]{\par\noindent{\bf Proof #1. }}{\hfill$\square$\\[2mm]}

\usepackage{fancyhdr}
\begin{document} 

\title{\vspace{-10mm} 
\begin{minipage}{1\textwidth}
\centering
\Large\bf Geometry and Expressive Power of\\Conditional Restricted Boltzmann Machines
\end{minipage}}

\author[1]{Guido Mont\'ufar} 
\author[1,2,3]{Nihat Ay} 
\author[1]{Keyan Ghazi-Zahedi} 
\affil[1]{\small Max Planck Institute for Mathematics in the Sciences, Inselstra\ss e 22, 04103 Leipzig, Germany}
\affil[2]{\small Department of Mathematics and Computer Science, Leipzig University, PF 10 09 20, 04009 Leipzig, Germany }
\affil[3]{\small Santa Fe Institute, 1399 Hyde Park Road, Santa Fe, NM 87501, USA}
\date{}

\maketitle

\begin{abstract}%
Conditional restricted Boltzmann machines are undirected stochastic neural networks with a layer of input and output units connected bipartitely to a layer of hidden units. These networks define models of conditional probability distributions on the states of the output units given the states of the input units, parametrized by interaction weights and biases. We address the representational power of these models, proving results their ability to represent conditional Markov random fields and conditional distributions with restricted supports, the minimal size of universal approximators, the maximal model approximation errors, and on the dimension of the set of representable conditional distributions. 
We contribute new tools for investigating conditional probability models, which allow us to improve the results that can be derived from existing work on restricted Boltzmann machine probability models.
\end{abstract}
\smallskip


\thispagestyle{empty}

\section{Introduction}

Restricted Boltzmann Machines (RBMs)~\citep{Smolensky1986,freund1994unsupervised} are generative probability models defined by undirected stochastic networks with bipartite interactions between visible and hidden units. 
These models are well-known in machine learning applications, where they are used to infer distributed representations of data and to train the layers of deep neural networks~\citep{hinton2006fast,Bengio:2009:LDA:1658423.1658424}. 
The restricted connectivity of these networks allows to train them efficiently on the basis of cheap inference and finite Gibbs sampling~\citep{Hinton:2002:TPE:639729.639730,HintonAParactical}, even when they are defined with many units and parameters. 
An RBM defines Gibbs-Boltzmann probability distributions over the observable states of the network, depending on the interaction weights and biases. An introduction is offered by~\citet{FisherAnIntroduction}.
The expressive power of these probability models has attracted much attention and has been studied in numerous papers, treating, in particular, their universal approximation properties~\citep[][]{Younes1996109,LeRoux:2008:RPR:1374176.1374187,Montufar2011}, approximation errors~\citep{NIPS2011_4380}, efficiency of representation~\citep{NIPS2013_5020,montufar2012does}, and dimension~\citep{Cueto2010}. 

In certain applications, it is preferred to work with conditional probability distributions, instead of joint probability distributions.
For example, in a classification task, the conditional distribution may be used to indicate a belief about the class of an input, without modeling the probability of observing that input; in sensorimotor control, it can describe a stochastic policy for choosing actions based on world observations; and in the context of information communication, to describe a channel. RBMs naturally define models of conditional probability distributions, called conditional restricted Boltzmann machines (CRBMs). These models inherit many of the nice properties of RBM probability models, such as the cheap inference and efficient training. Specifically, a CRBM is defined by clamping the states of an {\em input} subset of the visible units of an RBM. For each input state one obtains a conditioned distribution over the states of the {\em output} visible units. See Figure~\ref{figure:CRBM} for an illustration of this architecture. This kind of conditional models and slight variants thereof have seen success in many applications; for example, in classification~\citep{LarochelleB08}, collaborative filtering~\citep{Salakhutdinov:2007:RBM}, motion modeling~\citep{Taylor06modelinghuman,zeilerPigeon,DBLP:journals/corr/abs-1202-3748,sutskever_hinton_07}, and reinforcement learning~\citep[][]{Sallans:2004:RLF:1005332.1016794}.

So far, however, there is not much theoretical work addressing the expressive power of CRBMs.
We note that it is relatively straightforward to obtain some results on the expressive power of CRBMs from the existing theoretical work on RBM probability models. 
Nevertheless, an accurate analysis requires to take into account the specificities of the conditional case. 
Formally, a CRBM is a collection of RBMs, with one RBM for each possible input value. These RBMs differ in the biases of the hidden units, as these are influenced by the input values. 
However, these hidden biases are not independent for all different inputs, and, moreover, the same interaction weights and biases of the visible units are shared for all different inputs. This sharing of parameters draws a substantial distinction of CRBM models from independent tuples of RBM models. 

In this paper we address the representational power of CRBMs, contributing theoretical insights to the optimal number of hidden units. Our focus lies on the classes of conditional distributions that can possibly be represented by a CRBM with a fixed number of inputs and outputs, depending on the number of hidden units. Having said this, we do not discuss the problem of finding the optimal parameters that give rise to a desired conditional distribution (although our derivations include an algorithm that does this), nor problems related to incomplete knowledge of the target conditional distributions and generalization errors. A number of training methods for CRBMs have been discussed in the references listed above, depending on the concrete applications. 
The problems that we deal with here are the following: 
1) are distinct parameters of the model mapped to distinct conditional distributions; 
what is the smallest number of hidden units that suffices for obtaining a model that can 
2) approximate any target conditional distribution arbitrarily well (a universal approximator); 
3) approximate any target conditional distribution without exceeding a given error tolerance; 
4) approximate selected classes of conditional distributions arbitrarily well? 
We provide non-trivial solutions to all of these problems. 
We focus on the case of binary units, but the main ideas extend to the case of discrete non-binary units. 

\medskip

This paper is organized as follows. 
Section~\ref{sec:definitions} contains formal definitions and elementary properties of CRBMs. 
Section~\ref{sec:geometry} investigates the geometry of CRBM models in three subsections. 
In Section~\ref{sec:dimension} we study the dimension of the sets of conditional distributions represented by CRBMs and show that in most cases this is the dimension expected from counting parameters (Theorem~\ref{theorem:dimension}). 
In Section~\ref{sec:universal_approximation} we address the universal approximation problem, deriving upper and lower bounds on the minimal number of hidden units that suffices for this purpose (Theorem~\ref{theorem:universal}). 
In Section~\ref{sec:approximation_errors} we analyze the maximal approximation errors of CRBMs (assuming optimal parameters) and derive an upper-bound for the minimal number of hidden units that suffices to approximate every conditional distribution within a given error tolerance (Theorem~\ref{theorem:errors}). 
Section~\ref{sec:expressive} investigates the expressive power of CRBMs in two subsections. 
In Section~\ref{sec:hierarchicalconditionals} we describe how CRBMs can represent natural families of conditional distributions that arise in Markov random fields. 
In Section~\ref{sec:bounded_support} we study the ability of CRBMs to approximate conditional distributions with restricted supports. This section addresses, especially, the approximation of deterministic conditional distributions (Theorem~\ref{proposition:universaldeterministic}). 
In Section~\ref{sec:discussion} we offer a discussion and an outlook. 
In order to present the main results in a concise way, we have deferred all proofs to the appendices. 
Nonetheless, we think that the proofs are interesting in their own right, and we have prepared them with a fair amount of detail.

\section{Definitions}
\label{sec:definitions}

We will denote the set of probability distributions on $\{0,1\}^n$ by $\Delta_n$. 
A probability distribution $p\in\Delta_n$ is a vector of $2^n$ non-negative entries $p(y)$, $y\in\{0,1\}^n$, adding to one, $\sum_{y\in\{0,1\}^n}p(y)=1$. 
The set $\Delta_n$ is a $(2^n-1)$-dimensional simplex in $\R^{2^n}$.

We will denote the set of conditional distributions of a variable $y\in\{0,1\}^n$,
given another variable $x\in\{0,1\}^k$, by $\Del{k}{n}$. 
A conditional distribution $p(\cdot|\cdot)\in \Del{k}{n}$ is a $2^k\times 2^n$ row-stochastic matrix with rows $p(\cdot|x)\in\Delta_{n}$, $x\in\{0,1\}^k$. 
The set $\Del{k}{n}$ is a $2^k(2^n-1)$-dimensional polytope in $\R^{2^k\times 2^n}$. 
It can be regarded as the $2^k$-fold Cartesian product $\Del{k}{n} = \Delta_n\times \cdots\times \Delta_n$, 
where there is one probability simplex $\Delta_n$ for each possible input state $x\in\{0,1\}^k$. 
We will use the abbreviation $[N]:=\{1,\ldots, N\}$, where $N$ is a natural number. 

\begin{figure} 
	\vspace{-.1cm}
	\begin{center}
		\tikzset{node distance=2cm, auto}
		\begin{tikzpicture}[scale=0.9, every node/.style={transform shape}]
		\node (H) {\unitlayer{z}{4}{m}};
		\node (I) [left of = H,  below of=H] {\unitlayer{x}{2}{k}};
		\node (O) [right of = H, below of=H] {\unitlayer{y}{2}{n}};
		
		\draw[->,  line width = .8pt] (I) to node {$V$} (H);
		\draw[<->, line width = .8pt] (H) to node {$W$} (O);
		
		\node (h) at (1.9,.35) {};
		\node (o) at (3,-1.65) {};
		\node (c) [node distance=.7cm, right of = h, above of =h] {};
		\node (b) [node distance=.7cm, right of = o, above of =o] {};
		\draw[<-,  line width = .8pt] (h) to node {$c$} (c);
		\draw[<-,  line width = .8pt] (o) to node {$b$} (b);
		
		\node (i) [node distance=.8cm, below of = I] {Input layer};
		\node (h) [node distance=.8cm, above of = H] {Hidden layer};
		\node (o) [node distance=.8cm, below of = O] {Output layer};
		\end{tikzpicture}
		\vspace{-.3cm}
	\end{center}
	\caption{Architecture of a CRBM. An RBM is the special case with $k=0$. }\label{figure:CRBM}
\end{figure}
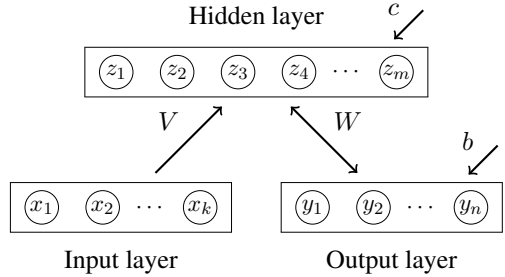

\begin{definition}\label{definitionCRBM}\normalfont
	The conditional restricted Boltzmann machine (CRBM) with $k$ input units, $n$ output units, and $m$ hidden units, denoted $\RBM_{n,m}^k$, is the set of all conditional distributions in $\Del{k}{n}$ that can be written as
	\begin{equation*}
	p(y|x) =  \frac{1}{Z(W,b,Vx+c)} \sum_{z\in\{0,1\}^m} \exp(z^\top V x + z^\top W y + b^\top y
	+ c^\top z),
	\quad \forall y\in\{0,1\}^n, x\in\{0,1\}^k,
	\label{definitioneq}
	\end{equation*}
	with normalization function
	\begin{equation*}
	Z(W,b,Vx+c) = \\
	\sum_{y\in\{0,1\}^n}\sum_{z\in\{0,1\}^m} \exp(z^\top V x + z^\top W y + b^\top y
	+ c^\top z), \\
	\quad \forall x\in\{0,1\}^k.
	\end{equation*}
	Here, $x$, $y$, and $z$ are column state vectors of the $k$ input units, $n$ output units, and $m$ hidden units, respectively, and ${}^\top$ denotes transposition. 
	The parameters of this model are the matrices of interaction weights $V\in\R^{m\times k}$, $W\in\R^{m\times n}$ and the vectors of biases $b\in\R^{n}$, $c\in\R^{m}$. 
	
	When there are no input units ($k=0$), the model $\RBM_{n,m}^k$ reduces to the restricted Boltzmann machine probability model with $n$ visible units and $m$ hidden units, denoted $\RBM_{n,m}$.
\end{definition}

We can view $\RBM_{n,m}^k$ as a collection of $2^k$ restricted Boltzmann machine probability models with shared parameters.
For each input $x\in\{0,1\}^k$, the output distribution $p(\cdot|x)$ is the probability distribution represented by $\RBM_{n,m}$ for the parameters $W,b,(Vx+c)$. 
All $p(\cdot|x)$ have the same interaction weights $W$, the same biases $b$ for the visible units, and differ only in the biases $(Vx+c)$ for the hidden units. 
The joint behavior of these distributions with shared parameters is not trivial. 

The model $\RBM_{n,m}^k$ can also be regarded as representing block-wise normalized versions of the joint probability distributions represented by $\RBM_{n+k,m}$. 
Namely, a joint distribution $p\in \RBM_{n+k,m}\subseteq \Delta_{k+n}$ is an array with entries $p(x,y)$, $x\in\{0,1\}^k$, $y\in\{0,1\}^n$. 
Conditioning $p$ on $x$ is equivalent to considering the normalized $x$-th row $p(y|x) = p(x,y) / \sum_{y'}p(x,y')$, $y\in\{0,1\}^n$.

\section{Geometry of Conditional Restricted Boltzmann Machines}
\label{sec:geometry}

In this section we investigate three basic questions about the geometry of CRBM models. 
First, what is the dimension of a CRBM model? 
Second, how many hidden units does a CRBM need in order to be able to approximate every conditional distribution arbitrarily well?  
Third, how accurate are the approximations of a CRBM, depending on the number of hidden units? 

\subsection{Dimension}
\label{sec:dimension}

The model $\RBM_{n,m}^k$ is defined by marginalizing out the hidden units of a graphical model. 
This implies that several choices of parameters may represent the same conditional distributions. 
In turn, the dimension of the set of representable conditional distributions may be smaller than the number of model parameters, in principle. 

When the dimension of $\RBM_{n,m}^k$ is equal to the number of parameters, $\dim(\RBM_{n,m}^k) = (k+n)m+n+m$, 
or, otherwise, equal to the dimension of the ambient polytope of conditional distributions, $\dim(\RBM_{n,m}^k) = 2^k(2^n-1)$, 
then the model is said to have the {\em expected dimension}. 
In this section we show that $\RBM_{n,m}^k$ has the expected dimension for most triplets $(k,n,m)$. 
In particular, we show that this holds in all practical cases, where the number of hidden units $m$ is smaller than exponential with respect to the number of visible units $k+n$. 

The dimension of a parametric model is given by the maximum of the rank of the Jacobian of its parametrization (assuming mild differentiability conditions). 
Computing the rank of the Jacobian is not easy in general. 
A resort is to compute the rank only in the limit of large parameters, which corresponds to considering a piece-wise linearized version of the original model, called the {\em tropical model}. 
\citet{Cueto2010} used this approach to study the dimension of RBM probability models. 
Here we apply their ideas in order to study the dimension of CRBM conditional models. 

The following functions from coding theory will be useful for phrasing the results: 
\begin{definition}
	\label{definition:ak}
	\normalfont
	Let $A(n, d)$ denote the cardinality of the largest subset of $\{0,1\}^n$ whose elements are at least Hamming distance $d$ apart.
	Let $K(n,d)$ denote the smallest cardinality of a set such that every element of $\{0,1\}^n$ is at most Hamming distance $d$ apart from that set. 
\end{definition}

\citet[][]{Cueto2010} showed that 
$\dim(\RBM_{n,m}) = nm +n +m$ for $m + 1\leq A(n, 3)$, and $\dim(\RBM_{n,m}) = 2^n -1$ for $m \geq K(n,1)$. 
It is known that $A(n,3)\geq 2^{n - \lceil \log_2(n+1) \rceil}$ and $K(n,1)\leq 2^{n-\lfloor \log_2(n+1) \rfloor}$. 
In turn, the probability model $\RBM_{n,m}$ has the expected dimension for most pairs $(n,m)$. 
Noting that $\dim(\RBM_{n,m}^k)\geq \dim(\RBM_{k+n,m})-(2^k-1)$, 
we directly infer the following bounds for the dimension of conditional models: 

\begin{proposition}
	\label{proposition:naivedim}
	\mbox{}
	\begin{itemize}
		\item
		$\dim(\RBM_{n,m}^k)\geq (n+k)m + n + m  + k-(2^k-1)$ for $m+1\leq A(k+n,3)$.
		\item
		$\dim(\RBM_{n,m}^k) = 2^{k}(2^n-1)$ for $m\geq K(k+n,1)$.
	\end{itemize}
\end{proposition}

These bounds are too loose  and do not allow us to attest whether the conditional model has the expected dimension, 
unless $m\geq K(k+n,1)$. 
Hence we need to study the conditional model in more detail. We obtain the following result:

\begin{theorem}
	\label{theorem:dimension}
	The conditional model $\RBM_{n,m}^k$ has the expected dimension in the following cases:
	\begin{itemize}
		\item 
		$\dim(\RBM_{n,m}^k) = (k+n+1)m + n$ for $m+1\leq A(k+n,4)$. 
		\item 
		$\dim(\RBM_{n,m}^k) = 2^k (2^n -1)$ for $m\geq K(k+n,1)$. 
	\end{itemize}
\end{theorem}

We note the following practical version of the theorem, 
which results from inserting appropriate bounds on the functions $A$ and $K$:  

\begin{corollary}
	\label{corollary:dimension}
	The conditional model $\RBM_{n,m}^k$ has the expected dimension in the following cases:
	\begin{itemize}
		\item $\dim(\RBM_{n,m}^k) = (n+k+1)m + n$ for $m \leq 2^{(k+n) - \lfloor\log_2( (k+n)^2 -(k+n) +2 )\rfloor}$. 
		\item $\dim(\RBM_{n,m}^k) = 2^k (2^n -1)$ for $m \geq 2^{(k+n)-\lfloor \log_2(k+n+1) \rfloor}$. 
	\end{itemize}
\end{corollary}

These results show that, in all cases of practical interest, where $m$ is less than exponential in $k+n$, 
the dimension of the CRBM model is indeed equal to the number of model parameters. 
In all these cases, almost every conditional distribution that can be represented by the model 
is represented by at most finitely many different choices of parameters. 

On the other hand, the dimension alone is not very informative about the ability of a model to approximate target distributions. 
In particular, it may be that a high dimensional model covers only a tiny fraction of the set of all conditional distributions, or also that a low dimensional model can approximate any target conditional relatively well. 
We address the minimal dimension and number of parameters of a universal approximator in the next section. 
In the subsequent section we address the approximation errors depending on the number of parameters.

\subsection{Universal Approximation} 
\label{sec:universal_approximation}

In this section we ask for the smallest number of hidden units $m$ for which the model $\RBM_{n,m}^k$ can approximate every conditional distribution from $\Del{k}{n}$ arbitrarily well. 

Note that each conditional distribution $p(y|x)$ can be identified with the set of joint distributions of the form $r(x,y)=q(x) p(y|x)$, with strictly positive marginals $q(x)$. 
In particular, by fixing a marginal distribution, we obtain an identification of $\Del{k}{n}$ and a subset of $\Delta_{k+n}$. 
Figure~\ref{Figure2} illustrates this identification in the case $n=k=1$ and $q\equiv \frac{1}{2}$. 

\begin{figure}
	\centering
	\setlength{\unitlength}{1cm}
	\begin{picture}(5,5.25)(0,-.35)
	\put(0,0){\includegraphics[width=5cm]{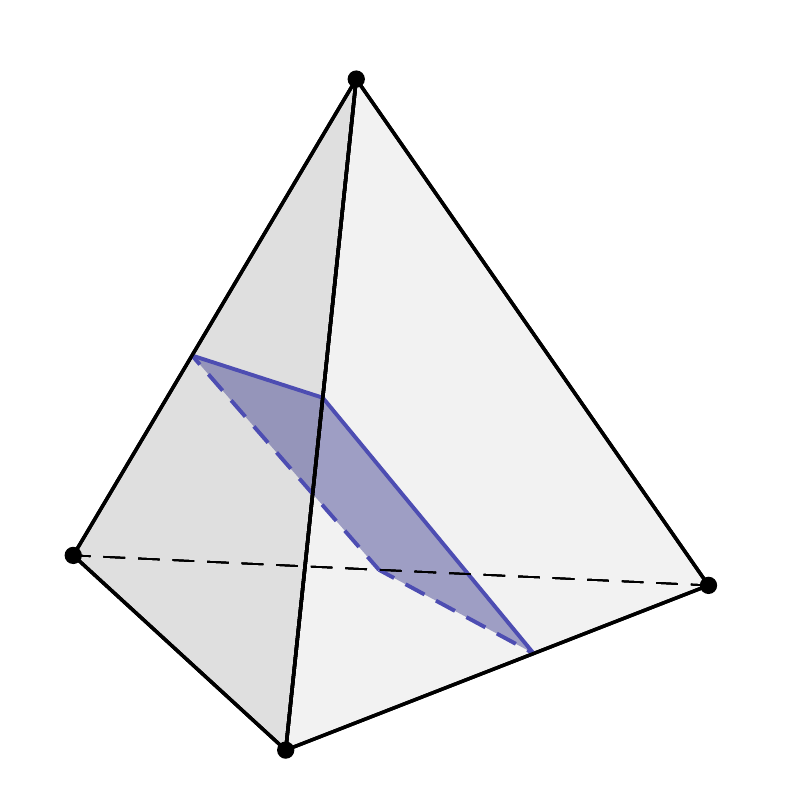}}
	\put(-.1,1.1){$\delta_{00}$}
	\put(1.5,-.25){$\delta_{10}$}
	\put(4.75,1){$\delta_{01}$}
	\put(2.5,4.5){$\delta_{11}$}
	\put(2.6,2){$\tfrac12\Del{1}{1}$}
	\put(3.75,3){$\Delta_{2}$}
	\end{picture}
	\caption{
		The polytope of conditional distributions~$\Del{1}{1}$ embedded in the probability simplex~$\Delta_2$.
	}\label{Figure2}
\end{figure}
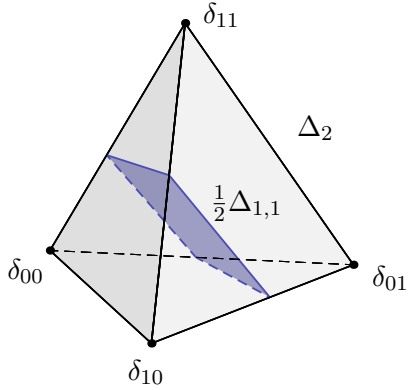

This implies that universal approximators of joint probability distributions define universal approximators of conditional distributions. We know that $\RBM_{n+k,m}$ is a universal approximator whenever $m\geq \frac{1}{2} 2^{k+n} -1$~\citep[see][]{Montufar2011}, and therefore:

\begin{proposition}
	\label{proposition:universal}
	The model $\RBM_{n,m}^k$ can approximate every conditional distribution from $\Del{k}{n}$ arbitrarily well whenever $m\geq \frac{1}{2}2^{k+n}-1$.
\end{proposition}

This improves previous results by~\citet{Younes1996109} and~\citet{Maaten2011}. 
On the other hand, since conditional models do not need to model the input-state distribution, 
in principle it is possible that $\RBM_{n,m}^k$ is a universal approximator even if $\RBM_{n+k,m}$ is not a universal approximator. 
In fact, we obtain the following improvement of Proposition~\ref{proposition:universal}, which does not follow from corresponding results for RBM probability models: 

\begin{theorem}
	\label{theorem:universal}
	The model $\RBM_{n,m}^k$ can approximate every conditional distribution from $\Del{k}{n}$ arbitrarily well whenever
	\begin{equation*}
	\renewcommand*{\arraystretch}{1.2}
	m\geq
	\left\{
	\begin{matrix*}[l]
	\frac{1}{2} 2^{k} (2^n-1) ,   	 	 & \text{if $k\geq 1$}\\
	\frac{3}{8} 2^{k} (2^n -1) +1  ,	 & \text{if $k\geq 3$}\\
	\frac{1}{4} 2^{k} (2^n-1 + 1/30) ,   & \text{if $k\geq 21$}
	\end{matrix*}
	\right. .
	\end{equation*}
	In fact, the model $\RBM_{n,m}^k$ can approximate every conditional distribution from $\Del{k}{n}$ arbitrarily well whenever
	$m \geq  2^{k} K(r) (2^n-1 ) + 2^{S(r)}P(r)$, where $r$ is any natural number satisfying $k\geq 1+\cdots +r =:S(r)$,
	and $K$ and $P$ are functions
	(defined in Lemma~\ref{lemma:starpacking} and Proposition~\ref{proposition:secondp}) which tend to approximately $0.2263$ and $0.0269$, respectively, as $r$ tends to infinity.
\end{theorem}

We note the following weaker but practical version of Theorem~\ref{theorem:universal}: 

\begin{corollary}
	\label{theorem:universal_shortversion}
	Let $k\geq 1$. The model $\RBM_{n,m}^k$ can approximate every conditional distribution from $\Del{k}{n}$ arbitrarily well whenever
	$m\geq \frac{1}{2} 2^{k} (2^n-1) = \frac{1}{2}2^{k+n} - \frac{1}{2}2^k$.
\end{corollary}

These results are significant, because they reduce the bounds following from universal approximation results for probability models by an additive term of order $2^k$, 
which corresponds precisely to the order of parameters needed in order to model the input-state distributions. 

As expected, the asymptotic behavior of the theorem's bound is exponential in the number of input and output units.
This lies in the nature of the universal approximation property. A crude lower bound on the number of hidden units that suffices for universal approximation can be obtained by comparing the number of parameters of the model and the dimension of the conditional polytope:

\begin{proposition}
	\label{proposition:universallower}
	If the model $\RBM_{n,m}^k$ can approximate every conditional distribution from $\Del{k}{n}$ arbitrarily well, then necessarily
	$m \geq \frac{1}{{(n+k+1 )}}(2^k(2^n-1) -n)$.
\end{proposition}

The results presented above highlight the fact that CRBM universal approximation may be possible with a drastically smaller number of hidden units than RBM universal approximation, for the same number of visible units. 
However, even with these reductions the universal approximation property requires an enormous number of hidden units. 
In order to provide a more informative description of the approximation capabilities of CRBMs, 
in the next section we investigate how the maximal approximation error decreases as hidden units are added to the model.

\subsection{Maximal Approximation Errors}
\label{sec:approximation_errors}

From a practical perspective it is not necessary to approximate conditional distributions arbitrarily well, but fair approximations suffice. 
This can be especially important if the number of required hidden units grows disproportionately with the quality of the approximation. 
In this section we investigate the maximal approximation errors of CRBMs depending on the number of hidden units. Figure~\ref{figure:approximationerror} gives a schematic illustration of the maximal approximation error of a conditional model. 

\begin{figure}
	\centering
	\includegraphics{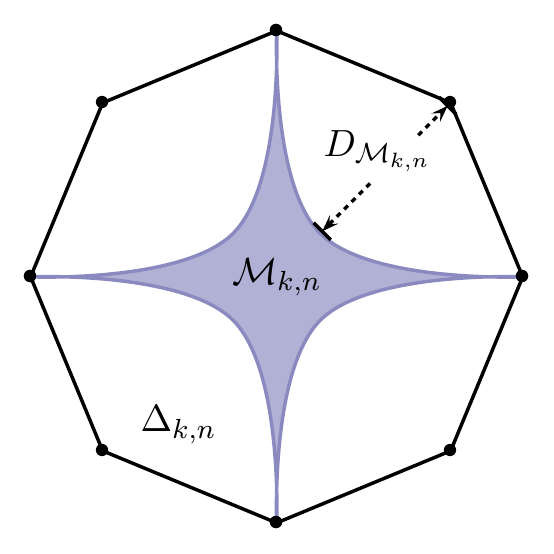}
	\caption{Schematic illustration of the maximal approximation error of a model $\Mcal_{k,n}\subseteq\Delta_{k,n}$. }
	\label{figure:approximationerror}
\end{figure}

The Kullback-Leibler divergence of two probability distributions $p$ and $q$ in $\Delta_{k+n}$ is given by
\begin{align*}
D(p \| q ) 
:=&  \sum_x \sum_y p(x) p(y|x) \log \frac{p(x) p(y|x)}{ q(x)q(y|x)}\\
= & D(p_X  \| q_X ) + \sum_x p(x) D(p(\cdot| x) \| q(\cdot\| x)),
\end{align*}
where $p_X=\sum_{y\in\{0,1\}^n} p(x,y)$ denotes the marginal distribution over $x\in\{0,1\}^k$. 

The divergence of two conditional distributions $p(\cdot|\cdot)$ and $q(\cdot|\cdot)$ in $\Del{k}{n}$ is given by 
\begin{equation*}
D(p(\cdot|\cdot) \| q(\cdot|\cdot) ) := \sum_{x} u_X(x) D(p(\cdot | x) \| q(\cdot|x)),
\end{equation*}
where $u_X$ denotes the uniform distribution over $x$. 
Even if the divergence between two joint distributions does not vanish, the divergence between their conditional distributions may vanish.  

The divergence from a conditional distribution $p(\cdot|\cdot)$ to the set $\Mcal_{k,n}$ of conditional distributions defined by a model of joint probability distributions $\Mcal_{k+n}$ is given by
\begin{equation*}
D(p(\cdot|\cdot)\| \Mcal_{k,n}) := 
\inf_{q\in\Mcal_{k,n}}  D(p(\cdot|\cdot) \| q(\cdot|\cdot) ) 
=
\inf_{q\in\Mcal_{k+n}}  D(u_X p(\cdot|\cdot) \| q )  -  D(u_X \| q_X). 
\end{equation*}

The maximum of the divergence from a conditional distribution to $\Mcal_{k,n}$ satisfies
\begin{equation*}
D_{\Mcal_{k,n}}
: = \max_{p(\cdot|\cdot)\in\Del{k}{n}} D(p(\cdot|\cdot)\|\Mcal_{k,n})
\leq \max_{p\in\Delta_{k+n}} D(p\|\Mcal_{k+n}) =: D_{{\Mcal}_{k+n}}.
\end{equation*}

Hence we can bound the maximal divergence of a CRBM by the maximal divergence of an RBM~\citep[studied in][]{NIPS2011_4380} and obtain the following: 
\begin{proposition}
	\label{proposition:errors}
	If $m\leq 2^{(n+k)-1}-1$, then the divergence from any conditional distribution $p(\cdot|\cdot)\in\Del{k}{n}$ to the model $\RBM_{n,m}^k$ is bounded by
	\begin{equation*}
	D_{\RBM_{n,m}^k}\leq D_{\RBM_{k+n,m}}
	\leq
	(n+k)-\lfloor\log_2(m+1)\rfloor -\frac{m+1}{2^{\lfloor\log_2(m+1)\rfloor}} .
	\end{equation*}
\end{proposition}

This proposition implies the universal approximation result from Proposition~\ref{proposition:universal} as the special case with vanishing approximation error, 
but it does not imply Theorem~\ref{theorem:universal} in the same way. 
Taking more specific properties of the conditional model into account, we can improve the proposition and obtain the following:  

\begin{theorem}
	\label{theorem:errors}
	Let $l\in [n]$.
	The divergence from any conditional distribution in $\Del{k}{n}$ to the model $\RBM_{n,m}^k$ is bounded from above by
	\begin{equation*}
	\renewcommand*{\arraystretch}{1.2}
	D_{\RBM_{n,m}^k}\leq n-l, \quad \text{whenever } m\geq
	\left\{
	\begin{matrix*}[l]
	\frac{1}{2} 2^{k} (2^l-1) ,   	 	 & \text{if $k\geq 1$}\\
	\frac{3}{8} 2^{k} (2^l -1) +1  ,	 & \text{if $k\geq 3$}\\
	\frac{1}{4} 2^{k} (2^l-1 + 1/30) ,   & \text{if $k\geq 21$}
	\end{matrix*}
	\right. .
	\end{equation*}
	In fact, the divergence from any conditional distribution in $\Del{k}{n}$ to $\RBM_{n,m}^k$ is bounded from above by
	$D_{\RBM_{n,m}^k} \leq n-l$, where $l$ is the largest integer with $m\geq 2^{k-S(r)}F(r)(2^l-1)+R(r)$.
\end{theorem}

This theorem implies the universal approximation result from Theorem~\ref{theorem:universal} as the special case with vanishing approximation error. 
We note the following weaker but practical version of Theorem~\ref{theorem:errors} (analogue to Corollary~\ref{theorem:universal_shortversion}):

\begin{corollary}
	Let $k\geq 1$ and $l\in [n]$.
	The divergence from any conditional distribution in $\Del{k}{n}$ to the model $\RBM_{n,m}^k$ is bounded from above by
	$D_{\RBM_{n,m}^k}\leq n-l$, whenever  $m\geq \frac{1}{2} 2^{k} (2^l-1)$.
\end{corollary}

Given an error tolerance, we can use these bounds to find a sufficient number of hidden units that guarantees approximations within this error tolerance. 

In plain terms,  the results presented above show that the worst case approximation errors of CRBMs decrease at least with the logarithm of the number of hidden units. On the other hand, in practice one is not interested in approximating all possible conditional distributions, but only special classes. 
One can expect that CRBMs can approximate certain classes of conditional distributions better than others. 
This is the subject of the next section.

\section{Representation of Special Classes of Conditional Models}
\label{sec:expressive}

In this section we ask about the classes of conditional distributions that can be compactly represented by CRBMs and whether CRBMs can approximate interesting conditional distributions using only a moderate number of hidden units. 

The first part of the question is about familiar classes of conditional distributions that can be expressed in terms of CRBMs, 
which in turn would allow us to compare CRBMs with other models and to develop a more intuitive picture of Definition~\ref{definitionCRBM}. 

The second part of the question clearly depends on the specific problem at hand. 
Nonetheless, some classes of conditional distributions may be considered generally interesting, as they contain solutions to all instances of certain classes of problems. 
An example is the class of deterministic conditional distributions, which suffices to solve any Markov decision problem in an optimal way. 

\subsection{Representation of Conditional Markov Random Fields}
\label{sec:hierarchicalconditionals}

In this section we discuss the ability of CRBMs to represent conditional Markov random fields, depending on the number of hidden units that they have. 
The main idea is that each hidden unit of an RBM can be used to model the pure interaction of a group of visible units. 
This idea appeared in previous work by~\citet{Younes1996109}, in the context of universal approximation. 

\begin{definition}
	Consider a simplicial complex $I$ on $[N]$; that is, a collection of subsets of $[N]=\{1,\ldots, N\}$ such that $A\in I$ implies $B\in I$ for all $B\subseteq A$, and $\emptyset\in I$. 
	The random field $\Ecal_I\subseteq\Delta_N$ with interactions $I$ is the set of probability distributions of the form 
	\begin{equation*}
	p(x) = \frac{1}{Z}\exp\Big( \sum_{A\in I} \theta_A \prod_{i\in A}x_i \Big), \quad\text{for all $x = (x_1,\ldots, x_N) \in \{0,1\}^N$}, 
	\end{equation*} 
	with normalization $Z=\sum_{x'\in\{0,1\}^N}\exp(\sum_{A\in I} \theta_A\prod_{i\in A}x'_i)$ and parameters $\theta_A \in \mathbb{R}$, $A\in I$. 
\end{definition}

\begin{figure}
	\centering
	\includegraphics[scale=.7]{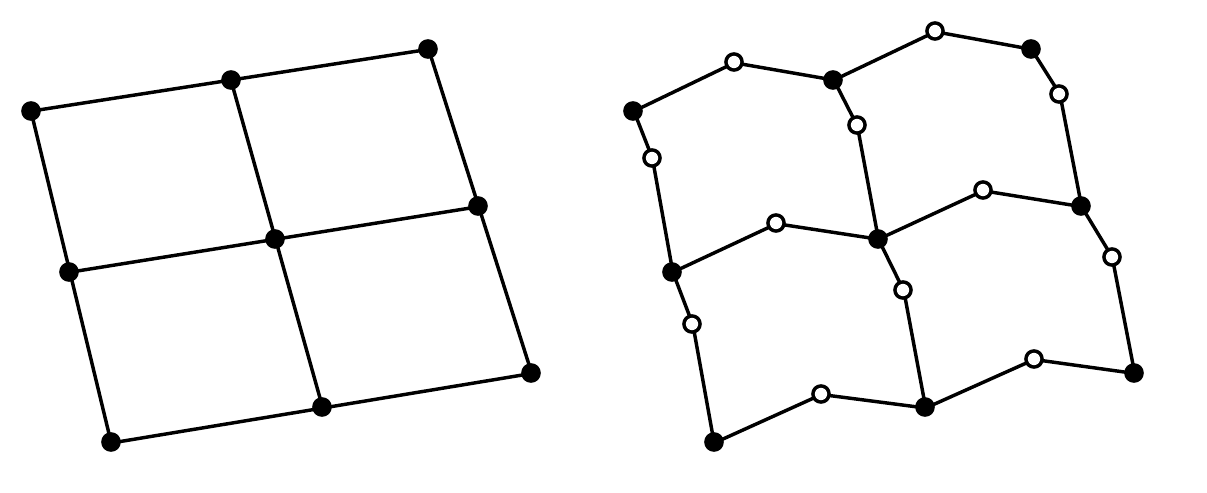}
	\caption{Example of a Markov random field and a corresponding RBM architecture that can represent it. }
	\label{figure:interactions}
\end{figure}

We obtain the following result:
\begin{theorem}
	\label{theorem:interaction}
	Let $I$ be a simplicial complex on $[k+n]$. 
	If $m \geq |\{A\in I\colon A\not\subseteq [k], A\neq \{k+1\},\ldots, A\neq \{k+n\}, A\neq \emptyset  \}|$, 
	then the model $\RBM^k_{n,m}$ can represent every conditional distribution of $(x_{k+1},\ldots, x_{k+n})$, given $(x_1,\ldots, x_k)$, that can be represented by $\Ecal_I\subseteq\Delta_{k+n}$. 
\end{theorem}

An interesting special case is when each output distribution can be chosen arbitrarily from a given Markov random field: 

\begin{corollary}
	\label{corollary:interaction}
	Let $I$ be a simplicial complex on $[n]$ and for each $x\in\{0,1\}^n$ 
	let $p^x$ be some probability distribution from $\Ecal_I\subseteq\Delta_n$. 
	If $m \geq 2^k(|I| - 1) - |\{A\in I\colon |A|=1 \}|$, 
	then the model $\RBM^k_{n,m}$ can represent the conditional distribution defined by $q(y|x) = p^x(y)$, for all $y\in\{0,1\}^n$, for all $x\in\{0,1\}^k$. 
\end{corollary}

We note the following direct implication for RBM probability models:  

\begin{corollary}
	Let $I$ be a simplicial complex on $[n]$. 
	If $m\geq |\{A\in I\colon |A|>1 \}|$, then $\RBM_{n,m}$ can represent any probability distribution $p$ from $\Ecal_I$. 
\end{corollary}

Figure~\ref{figure:interactions} illustrates a Markov random field and an RBM architecture that can represent it.

\subsection{Approximation of Conditional Distributions with Restricted Supports}
\label{sec:bounded_support}

In this section we continue the discussion about the classes of conditional distributions that can be represented by CRBMs, depending on the number of hidden units. 
Here we focus on a hierarchy of conditional distributions defined by the total number of input-output pairs with positive probability. 

\begin{definition}
	For any $k$, $n$, and $0\leq d\leq 2^k(2^n-1)$, 
	let $C_{k,n}(d) \subseteq \Delta_{k,n}$ denote the union of all $d$-dimensional faces of $\Delta_{k,n}$; that is, 
	the set of conditional distributions that have a total of $2^k+d$ or fewer non-zero entries, $C_{k,n}(d) := \{ p(\cdot|\cdot)\in\Delta_{k,n} \colon |\{(x,y)\colon p(y|x)>0 \}|\leq 2^k+d \}$. 
\end{definition}
Note that $C_{k,n}(2^k(2^n-1)) = \Delta_{k,n}$. 
The vertices (zero-dimensional faces) of $\Delta_{k,n}$ are the conditional distributions which assign positive probability to only one output, given each input, and are called \emph{deterministic}. 
By Carath\'eodory's theorem, every element of $C_{k,n}(d)$ is a convex combination of $(d+1)$ or fewer deterministic conditional distributions. 

The sets $C_{k,n}(d)$ arise naturally in the context of reinforcement learning and partially observable Markov decision processes (POMDPs). 
Namely, every finite POMDP has an associated effective dimension $d$, 
which is the dimension of the set of all state processes that can be generated by stationary stochastic policies. 
\citet[][]{montufar2014framework} showed that the policies represented by conditional distributions from the set $C_{k,n}(d)$ are sufficient to generate all the processes that can be generated by $\Delta_{k,n}$. 
In general, the effective dimension $d$ is relative small, 
such that $C_{k,n}(d)$ is a much smaller policy search space than $\Delta_{k,n}$. 

We have the following result: 
\begin{proposition}
	\label{proposition:boundedsupport}
	If $m\geq 2^k + d-1$, then the model $\RBM_{n,m}^k$ can approximate every element from $C_{k,n}(d)$ arbitrarily well. 
\end{proposition}

This result shows the intuitive fact that each hidden unit of  can be used to model the probability of an input-output pair. 
Since each conditional distribution has $2^k$ input-output probabilities that are completely determined by the other probabilities (due to normalization), it is interesting to ask whether the amount of hidden units indicated in the proposition is strictly necessary. 
Further below, Theorem~\ref{proposition:universaldeterministic} will show that, indeed, hidden units are required for 
modeling the positions of the positive probability input-output pairs, even if their specific values do not need to be modeled. 

We note that certain structures of positive probability input-output pairs can be modeled with fewer hidden units than stated in Proposition~\ref{proposition:boundedsupport}. 
An simple example is the following direct generalization of Corollary~\ref{theorem:universal_shortversion}: 

\begin{proposition}
	\label{theorem:universalsubset}
	If $d$ is divisible by $2^k$ and $m\geq d /2$, 
	then the model $\RBM_{n,m}^k$ can approximate every element from $C_{k,n}(d)$ arbitrarily well, when the set of positive-probability outputs is the same for all inputs. 
\end{proposition}

In the following we will focus on deterministic conditional distributions. 
This is a particularly interesting and simple class of conditional distributions with restricted supports. 
It is well known that any finite Markov decision processes (MDPs) has an optimal policy defined by a stationary deterministic conditional distribution~\citep[see][]{Bellman:1957,Ross:1983:ISD:538843}. 
Furthermore, ~\citet{ICCN} showed that it is always possible to define simple two-dimensional manifolds that approximate all deterministic conditional distributions arbitrarily well. 

Certain classes of conditional distributions (in particular deterministic conditionals) coming from feedforward networks can be approximated arbitrarily well by CRBMs:  

\begin{theorem}
	\label{theorem:ltn}
	The model $\RBM^k_{n,m}$ can approximate every conditional distribution arbitrarily well, which can be represented by a feedforward network with $k$ input units, a hidden layer of $m$ linear threshold units, and an output layer of $n$ sigmoid units. 
	In particular, the model $\RBM^k_{n,m}$ can approximate every deterministic conditional distribution from $\Delta_{k,n}$ arbitrarily well, 
	which can be represented by a feedforward linear threshold network with $k$ input, $m$ hidden, and $n$ output units. 
\end{theorem}

The representational power of feedforward linear threshold networks has been studied intensively in the literature. 
For example,~\citet{Wenzel:2000:HAS:361159.361171}  showed that a feedforward linear threshold network with $k\geq 1$ input, $m$ hidden, and $n=1$ output units, can represent the following: 
\begin{itemize}
	\item Any Boolean function $f\colon \{0,1\}^k\to\{0,1\}$, 
	when $m\geq 3\cdot 2^{k-1 -\lfloor \log_2(k+1) \rfloor }$; e.g., when $m\geq \frac{3}{k+2}2^k$. 
	\item The parity function $f_{\text{parity}}\colon \{0,1\}^k \to\{0,1\}; \; x\mapsto \sum_{i} x_i \mod 2$, when $m\geq k$. 
	\item The indicator function of any union of $m$ linearly separable subsets of $\{0,1\}^k$. 
\end{itemize}

Although CRBMs can approximate this rich class of deterministic conditional distributions arbitrarily well, 
the next result shows that the number of hidden units required for universal approximation of deterministic conditional distributions is rather large:

\begin{theorem}
	\label{proposition:universaldeterministic}
	The model $\RBM_{n,m}^k$ can approximate every deterministic policy from $\Del{k}{n}$ arbitrarily well
	if $m\geq \min \left\{ 2^k-1, \frac{3 n}{k+2} 2^k\right\}$ and only if $m\geq 2^{k/2} -\frac{(n+k)^2}{2n}$.
\end{theorem}

By this theorem, in order to approximate all deterministic conditional distributions arbitrarily well, a CRBM requires exponentially many hidden units, with respect to the number of input units.

\section{Conclusion}
\label{sec:discussion}

This paper gives a theoretical description of the representational capabilities of 
conditional restricted Boltzmann machines (CRBMs) relating model complexity and model accuracy. 
CRBMs are based on the well studied restricted Boltzmann machine (RBM) probability models. 
We proved an extensive series of results that generalize recent theoretical work on the representational power of RBMs in a non-trivial way. 

We studied the problem of parameter identifiability. 
We showed that every CRBM with up to exponentially many hidden units (in the number of input and output units) represent a set of conditional distributions of dimension equal to the number of model parameters. 
This implies that in all practical cases, CRBMs do not waste parameters, and, generically, 
only finitely many choices of the interaction weights and biases produce the same conditional distribution. 

We addressed the classical problems of universal approximation and approximation quality. 
Our results show that a CRBM with $m$ hidden units can approximate every conditional distribution of $n$ output units, given $k$ input units, without surpassing a Kullback-Leibler approximation error of the form $n - \log_2(m/2^{k-1} + 1)$ (assuming optimal parameters). 
Thus this model is a universal approximator whenever $m\geq \frac{1}{2}2^k(2^n-1)$. 
In fact we provided tighter bounds depending on $k$. 
For instance, if $k\geq 21$, then the universal approximation property is attained whenever $m\geq \frac{1}{4}2^k(2^n-29/30)$. 
Our proof is based on an upper bound for the complexity of an algorithm that packs Boolean cubes with sequences of non-overlapping stars, for which improvements may be possible. 
It is worth mentioning that the set of conditional distributions for which the approximation error is maximal may be very small. This is a largely open and difficult problem. 
We note that our results can be plugged into certain analytic integrals~\citep{montufar2014scaling} to produce upper-bounds for the expectation value of the approximation error when approximating conditional distributions drawn from a product Dirichlet density on the polytope of all conditional distributions. 
For future work it would be interesting to extend our (optimal-parameter) considerations by an analysis of the CRBM training complexity and the errors resulting from non-optimal parameter choices.

We also studied specific classes of conditional distributions that can be represented by CRBMs, depending on the number of hidden units. 
We showed that CRBMs can represent conditional Markov random fields by using each hidden unit to model the interaction of a group of visible variables. 
Furthermore, we showed that CRBMs can approximate all binary functions with $k$ input bits and $n$ output bits arbitrarily well if $m\geq 2^k-1$ or $m\geq \frac{3n}{k+2}2^k$ 
and only if $m\geq 2^{k/2}-(n+k)^2/2n$. 
In particular, this implies that there are exponentially many deterministic conditional distributions which can only be approximated arbitrarily well by a CRBM if the number of hidden units is exponential in the number of input units. 
This aligns with well known examples of functions that cannot be compactly represented by shallow feedforward networks, and reveals some of the intrinsic constraints of CRBM models that may prevent them from grossly over-fitting.

We think that the developed techniques can be used for studying other conditional probability models as well. 
In particular, for future work it would be interesting to compare the representational power of CRBMs and of combinations of CRBMs with feedforward nets (combined models of this kind include CRBMs with retroactive connections and recurrent temporal RBMs). 
Also, it would be interesting to apply our techniques to study stacks of CRBMs and other multilayer conditional models.
Finally, although our analysis focuses on the case of binary units, the main ideas can be extended to the case of discrete non-binary units.

\appendix

\section{Details on the Dimension}

\begin{proof}[of Proposition~\ref{proposition:naivedim}]
	Each joint distribution of $x$ and $y$ has the form $p(x,y)=p(x)p(y|x)$ and the set $\Delta_k$ of all marginals $p(x)$ has dimension $2^k-1$.
	This shows the first statement. The items follow directly from the corresponding statements for the probability model.
\end{proof}

\begin{proof}[of Theorem~\ref{theorem:dimension}]
	We will prove a stronger statement, where the condition on $m$ appearing in the first item is relaxed to the following: 
	The set $\{0,1\}^{k+n}$ contains $m$ disjoint radius-$1$ Hamming balls 
	whose union does not contain any set of the form $[x]:= \{(x,y) \in \{0,1\}^{k+n} \colon y\in\{0,1\}^n \}$ for $x\in\{0,1\}^k$, and whose complement has full affine rank as a subset of $\R^{k+n}$. 
	
	The proof is based on the ideas developed in~\citep{Cueto2010} for studying the RBM probability model. 
	
	We consider the Jacobian of $\RBM_{n,m}^k$ for the parametrization given in Definition~\ref{definitionCRBM}. 
	The dimension of $\RBM_{n,m}^k$ is the maximum rank of the Jacobian over all possible choices of $\theta = (W,V,b,c)\in\R^N$, $N= n+m+(n+k)m$. 
	Let $h_\theta(v) :=\argmax_{z\in\{0,1\}^m} p(z|v)$ denote the most likely hidden state of $\RBM_{k+n,m}$ given the visible state $v=(x,y)$, depending on the parameter $\theta$. 
	After a few direct algebraic manipulations, we find that the maximum rank of the Jacobian is bounded from below by the maximum over $\theta$ of the dimension of the column-span of the matrix 
	$\Acal_\theta$ with rows 
	\begin{equation}
	\left[ (1,  x^\top ,  y^\top )  , ( 1, x^\top,  y^\top ) \otimes h_\theta(x,y) ^\top \right], \quad\text{for all  $(x,y)\in\{0,1\}^{k+n}$} , 
	\end{equation}
	modulo vectors whose $(x,y)$-th entries are independent of $y$ given $x$.
	Here $\otimes$ is the Kronecker product, which is defined by $(a_{ij})_{i,j}\otimes (b_{kl})_{k,l}=(a_{ij}b_{kl})_{ik,jl}$.
	The modulo operation has the effect of disregarding the input distribution $p(x)$ in the joint distribution $p(x,y)=p(x)p(y|x)$ represented by the RBM.
	For example, from the first block of $\Acal_\theta$ we can remove the columns that correspond to $x$, without affecting the mentioned column-span.
	Summarizing, the maximal column-rank of $\Acal_\theta$ modulo the vectors whose $(x,y)$-th entries are independent of $y$ given $x$ is a lower bound for the dimension of $\RBM_{n,m}^k$.
	
	Note that $\Acal_\theta$ depends on $\theta$ in a discrete way; the parameter space $\R^N$ is partitioned in finitely many regions where $\Acal_\theta$ is constant.
	The piece-wise linear map thus emerging, with linear pieces represented by the $\Acal_\theta$, is the tropical CRBM morphism, and its image is the tropical CRBM model.
	
	Each linear region of the tropical morphism corresponds to an inference function $h_\theta\colon\{0,1\}^{k+n}\to\{0,1\}^m$ taking visible state vectors to the most likely hidden state vectors.
	Geometrically, such an inference function corresponds to $m$ slicings of the $(k+n)$-dimensional unit hypercube.
	Namely, every hidden unit divides the visible space $\{0,1\}^{k+n}\subset\R^{k+n}$ in two halfspaces, according to its preferred state.
	
	Each of these $m$ slicings defines a column block of the matrix $\mathcal{A}_\theta$. More precisely,
	\begin{equation*}
	\Acal_\theta = \left( A | A_{C_1} | \cdots | A_{C_m} \right),
	\end{equation*}
	where $A$ is the matrix with rows $(1, v_1,\ldots, v_{k+n})$ for all $v\in\{0,1\}^{k+n}$, and $A_C$ is the same matrix, with rows multiplied by the indicator function of the set $C$ of points $v$ classified as positive by a linear classifier (slicing).
	
	If we consider only linear classifiers that select rows of $A$ corresponding to disjoint Hamming balls of radius one (that is, such that the $C_i$ are disjoint radius-one Hamming balls),
	then the rank of $\Acal_\theta$ is equal to the number of such classifiers times $(n+k+1)$ (which is the rank of each block $A_{C_i}$),
	plus the rank of $A_{\{0,1\}^{k+n}\setminus\cup_{i\in[m]} C_i}$ (which is the remainder rank of the first block $A$).
	The column-rank modulo functions of $x$ is equal to the rank minus $k+1$ (which is the dimension of the functions of $x$ spanned by columns of $A$),
	minus at most the number of cylinder sets $[x]=\{(x,y)\colon y\in\{0,1\}^n \}$ for some $x\in\{0,1\}^k$ that are contained in $\cup_{i\in[m]}C_i$.
	This completes the proof of the general statement in the first item.
	
	The example given in the first item is a consequence of the following observations.
	Each cylinder set $[x]$ contains $2^n$ points. If a given cylinder set $[x]$ intersects a radius-$1$ Hamming ball $B$ but is not contained in it, then it also intersects the radius-$2$ Hamming sphere around $B$.
	Choosing the radius-$1$ Hamming ball slicings $C_1,\ldots, C_m$ to have centers at least Hamming distance $4$ apart, we can ensure that their union does not contain any cylinder set $[x]$.
	
	The second item is by the second item of Proposition~\ref{proposition:naivedim};
	when the probability model $\RBM_{n+k,m}$ is full dimensional, then $\RBM_{n,m}^k$ is full dimensional.
\end{proof}

\begin{proof}[of Corollary~\ref{corollary:dimension}]
	For the maximal cardinality of distance-$4$ binary codes of length $l$ it is known that
	\mbox{$A(l,4)\geq 2^r$}, where $r$ is the largest integer with $2^r < \frac{2^{l}}{ 1 + (l-1) + (l-1)(l-2)/2}$~\citep[][]{Gilbert:1952,Varshamov:1957},
	and so $A_{2}(l,4) \geq 2^{l - \lfloor\log_2( l^2 -l +2 )\rfloor}$.
	Furthermore, for the minimal size of radius one covering codes of length $l$ it is known that $K(l,1)\leq 2^{l-\lfloor \log_2(l+1) \rfloor}$~\citep[][]{Cueto2010}. 
\end{proof}

\section{Details on Universal Approximation}

\subsection{Sufficient Number of Hidden Units}
\label{subsec:main_result_universal}

This section contains the proof of Theorem~\ref{theorem:universal} about the minimal size of CRBM universal approximators.
The proof is constructive; given any target conditional distribution, it proceeds by adjusting the weights of the hidden units successively until obtaining the desired approximation.
The idea of the proof is that each hidden unit can be used to model the probability of an output vector, for several different input vectors.
The probability of a given output vector can be adjusted at will by a single hidden unit, jointly for several input vectors, when these input vectors are in general position.
This comes at the cost of generating dependent output probabilities for all other inputs in the same affine space.
The main difficulty of the proof lies in the construction of sequences of successively conflict-free groups of affinely independent inputs, and in estimating the shortest possible length of such sequences exhausting all possible inputs.
The proof is composed of several lemmas and propositions. 
We start with a few definitions:

\begin{definition}
	\label{definition:Hadamard}
	\normalfont
	Given two probability distributions $p$ and $q$ on a finite set $\Xcal$, the \emph{Hadamard product} or renormalized entry-wise product $p\ast q$
	is the probability distribution on $\Xcal$ defined by $(p\ast q)(x) = p(x)q(x)/ \sum_{x'}p(x')q(x')$ for all $x\in\Xcal$.
	When building this product, we assume that the supports of $p$ and $q$ are not disjoint, such that the normalization term does not vanish.
\end{definition}

The probability distributions that can be represented by RBMs can be described in terms of Hadamard products.
Namely, for every probability distribution $p$ that can be represented by $\RBM_{n,m}$,
the model $\RBM_{n,m+1}$ with one additional hidden unit can represent precisely the probability distribution of the form $p'=p\ast q$,
where $q=\lambda' r + (1-\lambda')s$ is a mixture, with $\lambda'\in[0,1]$, of two strictly positive product distributions $r(x)=\prod_{i=1}^n r_i(x_i)$ and $s(x)=\prod_{i=1}^n s_i(x_i)$. 
In other words, each additional hidden unit amounts to Hadamard-multiplying the distributions representable by an RBM with the distributions representable as mixtures of product distributions.
The same result is obtained by considering only the Hadamard products with mixtures where $r$ is equal to the uniform distribution.
In this case, the distributions $p'=p\ast q$ are of the form $p'=\lambda p + (1-\lambda)p\ast s$,
where $s$ is any strictly positive product distribution and $\lambda = \frac{\lambda'}{\lambda'  + 2^n (1-\lambda') \sum_{x}p(x)s(x)}$ is any weight in $[0,1]$.

\begin{definition}\normalfont
	A {\em probability sharing step} is a transformation taking a probability distribution $p$ to $p'=\lambda p + (1-\lambda)p\ast s$, for some strictly positive product distribution $s$ and some $\lambda\in[0,1]$.
\end{definition}

We will need two more standard definitions from coding theory:

\begin{definition}\normalfont
	A radius-$1$ {\em Hamming ball} in $\{0,1\}^k$ is a set $B$ consisting of a length-$k$ binary vector and all its immediate neighbors; that is,
	$B = \{ x\in\{0,1\}^k\colon d_H(x,z)\leq 1 \}$ for some $z\in\{0,1\}^k$, where $d_H(x,z):=|\{i\in[k] \colon x_i\neq z_i \}|$ denotes the Hamming distance between $x$ and~$z$. Here $[k]:=\{1,\ldots, k\}$.
\end{definition}

\begin{definition}
	\normalfont
	An $r$-dimensional {\em cylinder set} in $\{0,1\}^k$ is a set $C$ of length-$k$ binary vectors with arbitrary values in $r$ coordinates and fixed values in the other coordinates; that is,
	$C= \{x\in\{0,1\}^k\colon x_i =z_i \text{ for all $i\in\Lambda$}  \} $ for some $z\in\{0,1\}^k$ and some $\Lambda\subseteq [k]$ with $k-|\Lambda| =r$.
\end{definition}
The geometric intuition is simple: a cylinder set corresponds to the vertices of a face of a unit cube,
and a radius-$1$ Hamming ball corresponds to the vertices of a corner of a unit cube.
The vectors in a radius-$1$ Hamming ball are affinely independent.
See~Figure~\ref{figure:stars}A for an illustration.
\medskip

In order to prove Theorem~\ref{theorem:universal},
for each $k\in\N$ and $n\in\N$ we want to find an $m_{k,n}\in\N$ such that:
for any given strictly positive conditional distribution $q(\cdot|\cdot)$,
there exists $p\in\RBM_{n+k,0}$ and $m_{k,n}$ probability sharing steps taking $p$ to a strictly positive joint distribution $p'$
with $p'(\cdot|\cdot)=q(\cdot|\cdot)$.
The idea is that the starting distribution is represented by an RBM with no hidden units, and each sharing step is realized by adding a hidden unit to the RBM. In order to obtain these sequences of sharing steps, we will use the following technical lemma:

\begin{lemma}
	\label{lemma:sharing}
	Let $B$ be a radius-$1$ Hamming ball in $\{0,1\}^k$ and let $C$ be a cylinder subset of $\{0,1\}^k$ containing the center of $B$.
	Let $\lambda^x\in(0,1)$ for all $x\in B\cap C$, let $\tilde y\in\{0,1\}^n$ and let $\delta_{\tilde y}$ denote the Dirac delta on $\{0,1\}^n$ assigning probability one to $\tilde y$.
	Let $p\in\Delta_{k+n}$ be a strictly positive probability distribution with conditionals $p(\cdot|x)$ and let
	\begin{equation*}
	p'(\cdot|x) :=
	\begin{cases}
	\lambda^x p(\cdot|x) + (1-\lambda^x) \delta_{\tilde y}, & \text{for all $x\in B\cap C$}\\
	p(\cdot|x), & \text{for all $x\in \{0,1\}^k\setminus C$}
	\end{cases}.
	\end{equation*}
	Then, for any $\epsilon > 0$, there is a 
	probability sharing step taking $p$ to a joint distribution $p''$
	with conditionals satisfying $\sum_{y} |p''(y|x) - p'(y|x)|\leq \epsilon $ for all $x\in(B\cap C) \cup (\{0,1\}^k\setminus C)$.
\end{lemma}

\begin{proof}
	We define the sharing step $p'=\lambda p + (1-\lambda)p\ast s$ with a product distribution $s$ supported on $\{\tilde y\}\times C\subseteq \{0,1\}^{k+n}$.
	Note that given any distribution $q$ on $C$ and a radius-$1$ Hamming ball $B$ whose center is contained in $C$,
	there is a product distribution $s$ on $C$ such that $(s(x))_{x\in C\cap B} \propto (q(x))_{x\in C\cap B}$.
	In other words, the restriction of a product distribution $s$ to a radius-$1$ Hamming ball $B$ can be made proportional to any non-negative vector of length $|B|$.
	To see this, note that a product distribution is a vector with entries $s(x) = \prod_{i\in[k]} s_i(x_i)$ for all $x=(x_1,\ldots, x_k)$, with factor distributions $s_i$.
	Hence the restriction of $s$ to $B$ is given by the vector $\Big(\prod_i s_i(0), \frac{s_1(1)}{s_1(0)}\prod_{i} s_i(0), \ldots, \frac{s_k(1)}{s_k(0)}\prod_{i} s_i(0)\Big)$, where, without loss of generality, we chose $B$ centered at $(0,\ldots,0)$.
	Now, by choosing the factor distributions $s_i$ appropriately, the vector $\big(\frac{s_1(1)}{s_1(0)},\ldots, \frac{s_k(1)}{s_k(0)}\big)$ can be made arbitrary in $\R^k_+$.
\end{proof}

We have the following two implications of Lemma~\ref{lemma:sharing}:

\begin{corollary}
	\label{implication1}
	For any $\epsilon>0$ and $q(\cdot|x)\in\Delta_n$ for all $x\in B\cap C$,
	there is an $\epsilon'>0$ such that, for any strictly positive joint distribution $p\in\Delta_{k+n}$ with conditionals satisfying $\sum_y | p(y | x) -  \delta_0(y)| \leq \epsilon'$ for all $x\in B\cap C$,
	there are $2^n-1$ sharing steps taking $p$ to a joint distribution $p''$ with conditionals satisfying $\sum_{y} |p''(y|x)- p'(y|x)|\leq \epsilon$ for all $x\in (B\cap C)\cup (\{0,1\}^k\setminus C)$, where
	$\delta_0$ is the Dirac delta on $\{0,1\}^n$ assigning probability one to the vector of zeros and
	\begin{equation*}
	p'(\cdot|x) : =
	\begin{cases}
	q(\cdot|x),& \text{for all $x\in B\cap C$} \\
	p(\cdot|x), & \text{for all $x\in \{0,1\}^k\setminus C$}
	\end{cases} .
	\end{equation*}
\end{corollary}

\begin{proof}
	Consider any $x\in B\cap C$.
	We will show that the probability distribution $q(\cdot|x)\in \Delta_n$ can be written as the transformation of a Dirac delta by $2^n-1$ sharing steps.
	Then the claim follows from Lemma~\ref{lemma:sharing}.
	Let $\sigma\colon \{0,1\}^n \to \{0,\ldots,2^n-1\}$ be an enumeration of $\{0,1\}^n$.
	Let $p^{(0)}(y|x) =\delta_{ \sigma^{-1}(0)}(y)$ be the starting distribution (the Dirac delta concentrated at the state $\tilde y\in\{0,1\}^n$ with $\sigma(\tilde y) = 0$) and let the $t$-th sharing step be defined by
	$p^{(t)}(y) = \lambda^x_{\sigma^{-1}(t)} p^{(t-1)}(y|x) + (1-\lambda^x_{\sigma^{-1}(t)})\delta_{\sigma^{-1}(t)}(y)$, 
	for some weight $\lambda^x_{\sigma^{-1}(t)}\in[0,1]$.
	After $2^n-1$ sharing steps, we obtain the distribution
	\begin{equation*}
	p^{(2^n-1)}(y|x) = \sum_{\tilde y} \Big( \prod_{\tilde y'\colon \sigma(\tilde y') > \sigma(\tilde y)} \lambda^x_{\tilde y'}\Big) ( 1- \lambda^x_{\tilde y}) \delta_{\tilde y}(y), \quad\text{for all $y\in\{0,1\}^n$},
	\end{equation*}
	whereby $\lambda^x_{\tilde y}:=0$ for $\sigma(\tilde y)=0$.
	This distribution is equal to $q(\cdot|x)$ for the following choice of weights:
	\begin{equation*}
	\lambda^x_{\tilde y} : = 1 - \frac{q(\tilde y|x)} { 1 - \sum_{\tilde y'\colon \sigma(\tilde y')>\sigma(\tilde y) } q(\tilde y'|x)}, \quad \text{for all } \tilde y\in\{0,1\}^n.
	\end{equation*}
	It is easy to verify that these weights satisfy the condition $\lambda^x_{\tilde y}\in[0,1]$ for all $\tilde y\in\{0,1\}^n$, and $\lambda^x_{\tilde y} =0$ for that $\tilde y$ with $\sigma(\tilde y) =0$, independently of the specific choice of $\sigma$.
\end{proof}

Note that this corollary does not make any statement about the rows $p''(\cdot|x)$ with $x\in C\setminus B$.
When transforming the $(B\cap C)$-rows of $p$ according to Lemma~\ref{lemma:sharing},
the $(C\setminus B)$-rows get transformed as well, in a non-trivial dependent way.
Fortunately, there is a sharing step that allows us to ``reset'' exactly certain rows to a desired point measure, without introducing new non-trivial dependencies:

\begin{corollary}
	\label{implication2}
	For any $\epsilon>0$, any cylinder set $C\subseteq\{0,1\}^k$, and any $\tilde y\in\{0,1\}^n$,
	any strictly positive joint distribution $p$ can be transformed by a probability sharing step to a joint distribution $p''$ with conditionals satisfying $\sum_{y}|p''(y|x)-p'(y|x)|\leq \epsilon$ for all $x\in\{0,1\}^k$, where
	\begin{equation*}
	p'(\cdot|x) :=
	\begin{cases}
	\delta_{\tilde y},& \text{for all $x\in C$} \\
	p(\cdot|x), & \text{for all $x\in \{0,1\}^k\setminus C$}
	\end{cases}.
	\end{equation*}
\end{corollary}
\begin{proof}
	The sharing step can be defined as $p''= \lambda p + (1-\lambda)p\ast s$ with $s$ close to the uniform distribution on $\{\tilde y\}\times C$ and $\lambda$ close to $0$ (close enough depending on $\epsilon$).
\end{proof}
We will refer to a sharing step as described in Corollary~\ref{implication2} as a {\em reset} of the $C$-rows of $p$. \\

With all the observations made above, we can construct an algorithm that generates an arbitrarily accurate approximation of any given conditional distribution by applying a sequence of sharing steps to any given strictly positive joint distribution.
We denote by {\em star} the intersection of a radius-$1$ Hamming ball and a cylinder set containing the center of the ball.
See Figure~\ref{figure:stars}A.
The details of the algorithm are given in Algorithm~\ref{alg1}.

\begin{algorithm} [t]
	\SetAlgoLined
	\KwIn{Strictly positive joint distribution $p$, target conditional distribution $q(\cdot|\cdot)$, and $\epsilon>0$}
	\KwOut{Transformation $p'$ of the input distribution with $\sum_y|p'(y|x) - q(y|x)|\leq \epsilon$ for all $x$}
	Initialize $\Bcal \leftarrow \emptyset$\;
	\While{$\Bcal\not\supseteq \{0,1\}^k$}{
		Choose (disjoint) cylinder sets $C^1,\ldots, C^K$ packing $\{0,1\}^k\setminus \Bcal$\;
		If needed, perform at most $K$ sharing steps resetting the $C^i$ rows of $p$ for all $i\in[K]$, \newline
		taking $p(\cdot|x)$ close to $\delta_0$ for all $x\in C^i$ for all $i\in[K]$ and leaving all other rows close to their current values,
		according to Corollary~\ref{implication2}\; %
		\For{each $i\in[K]$}{
			Perform at most $2^n-1$ sharing steps taking $p(\cdot|x)$ close to $q(\cdot|x)$ for all $x\in B^i$, where $B^i$ is some star contained in $C^i$, and leaving the $(\{0,1\}^k\setminus C^i)$-rows close to their current values, according to Corollary~\ref{implication1}\;
		}
		$\Bcal\leftarrow \Bcal \cup (\cup_{i\in[K]} B^i)$\;
	}
	\caption{%
		Algorithmic illustration of the proof of Theorem~\ref{theorem:universal}.
		The algorithm performs sequential sharing steps on a strictly positive joint distribution $p\in\Delta_{k+n}$ until the resulting distribution $p'$ has a conditional distribution $p'(\cdot|\cdot)$ satisfying $\sum_y |p'(y|x) - q(y|x)|\leq \epsilon$ for all $x$.
		Here $\Bcal\subseteq\{0,1\}^k$ denotes the set of inputs $x$ that have been readily processed in the current iteration.
	}\label{alg1}
	\BlankLine
\end{algorithm}

In order to obtain a bound on the number $m$ of hidden units for which $\RBM_{n,m}^k$ can approximate
a given target conditional distribution arbitrarily well, we just need to evaluate the number of sharing steps run by Algorithm~\ref{alg1}.
For this purpose, we investigate the combinatorics of sharing step sequences and evaluate their worst case lengths.
We can choose as starting distribution some $p \in\RBM_{n+k,0}$ with conditionals satisfying $\sum_y |p(y|x) - \delta_0(y)|\leq \epsilon'$ for all $x\in\{0,1\}^k$, for some $\epsilon'>0$ small enough depending on the target conditional $q(\cdot|\cdot)$ and the targeted approximation accuracy $\epsilon$.

\begin{definition}
	\normalfont
	A sequence of stars $B^1,\ldots, B^l$ packing $\{0,1\}^k$ with the property that the smallest cylinder set containing any of the stars in the sequence does not intersect any previous star in the sequence
	is called a {\em star packing sequence} for $\{0,1\}^k$.
\end{definition}

The number of sharing steps run by Algorithm~\ref{alg1} is bounded from above by $(2^n-1)$ times the length of a star packing sequence for the set of inputs $\{0,1\}^k$.
Note that the choices of stars and the lengths of the possible star packing sequences are not unique.
Figure~\ref{figure:stars}B gives an example showing that starting a sequence with large stars is not necessarily the best strategy to produce a short sequence.
The next lemma states that there is a class of star packing sequences of a certain length, depending on the size of the input space.
Thereby, this lemma upper-bounds the worst case complexity of Algorithm~\ref{alg1}.

\begin{lemma}
	\label{lemma:starpacking}
	Let $r\in\N$,  $S(r):=1+2+\cdots+r$, $k\geq S(r)$,  $f_i(z):= 2^{S(i-1)} + (2^i -(i+1))z$, and $F(r):=f_r(f_{r-1}( \cdots f_2(f_1) ))$.
	There is a star packing sequence for $\{0,1\}^k$ of length
	$2^{k- S(r)} F(r)$. 
	Furthermore, for this sequence, Algorithm~\ref{alg1} requires at most $R(r):=\prod_{i=2}^{r}(2^{i}-(i+1))$ resets.
\end{lemma}

\begin{figure}[t]
	\centering
	A\hspace{.1cm}\includegraphics[clip=true,trim=0cm .5cm .2cm 0cm,scale=.8]{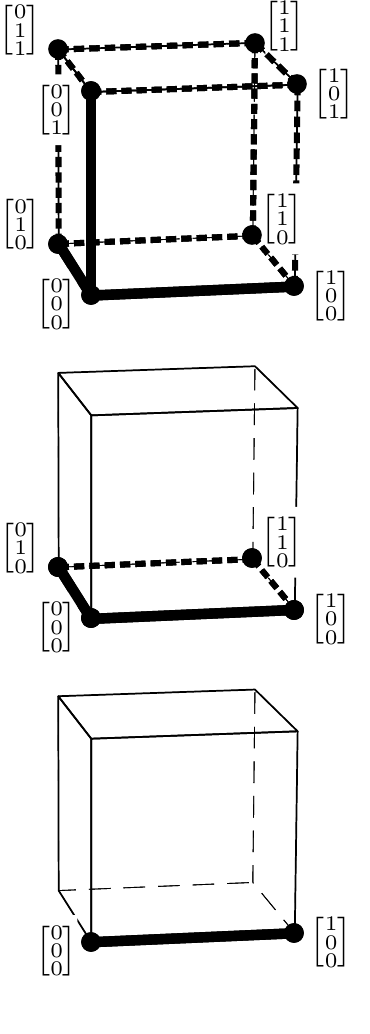}
	\;\;\;\;
	B\hspace{.1cm}\includegraphics[clip=true,trim=0cm .5cm .2cm 0cm,scale=.8]{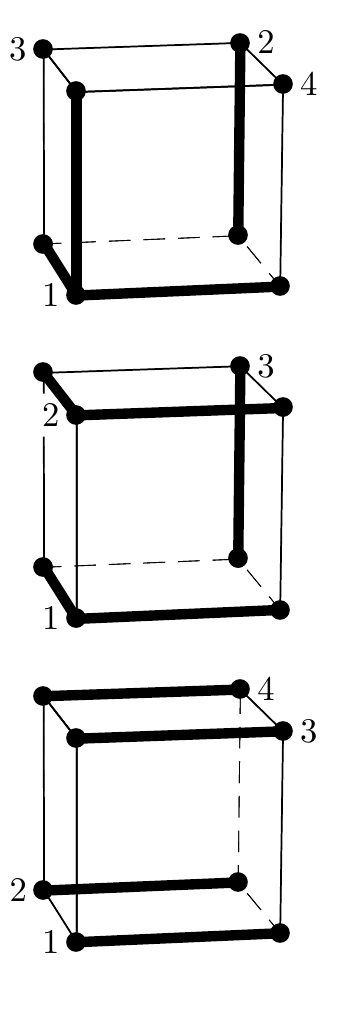}
	\;\;\;\;
	C\hspace{-.4cm}\includegraphics[clip=true,trim=0cm 0.8cm 1cm .6cm,scale=.9]{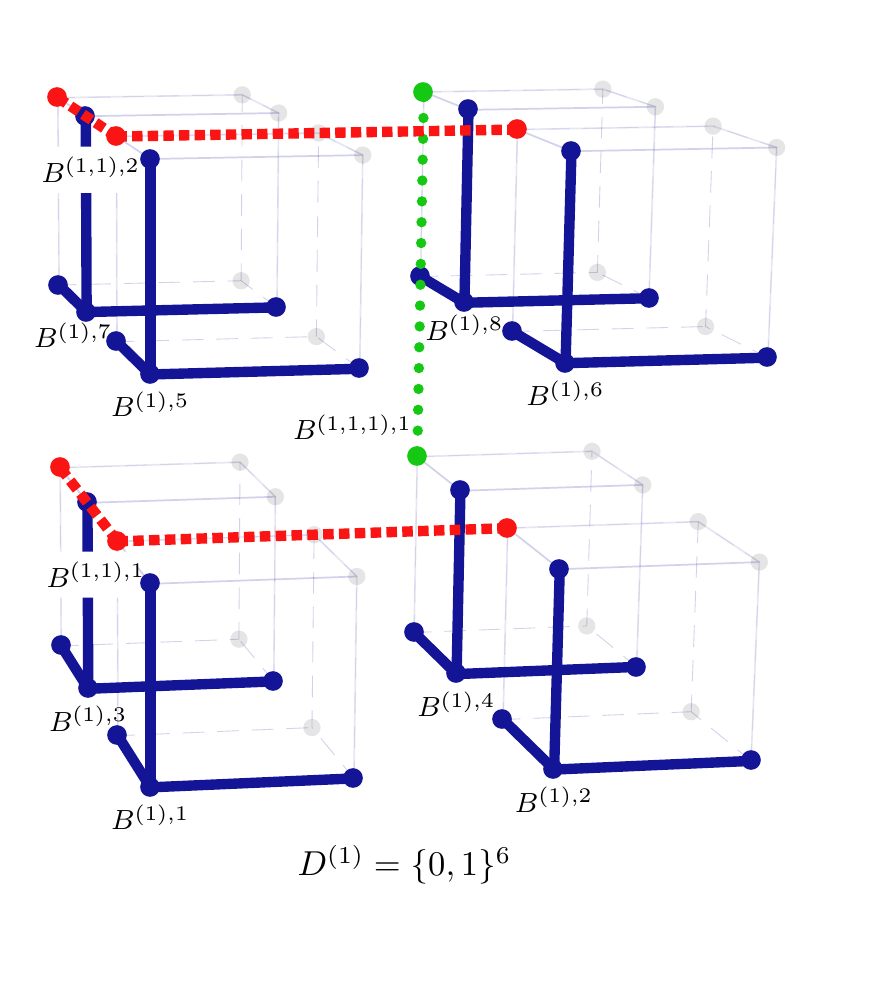} 
	\caption{
		A)~Examples of radius-$1$ Hamming balls in cylinder sets of dimension $3$, $2$, and $1$.
		The cylinder sets are shown as bold vertices connected by dashed edges, and the nested Hamming balls (stars) as bold vertices connected by solid edges.
		B)~Three examples of star packing sequences for $\{0,1\}^3$.
		C)~Illustration of the star packing sequence constructed in Lemma~\ref{lemma:starpacking} for $\{0,1\}^6$.
	}
	\label{figure:stars}
\end{figure}

\begin{proof}
	The star packing sequence is constructed by the following procedure.
	In each step, we define a set of cylinder sets packing all sites of $\{0,1\}^k$ that have not been covered by stars so far,
	and include a sub-star of each of these cylinder sets in the sequence.
	
	As an initialization step, we split $\{0,1\}^k$ into $2^{k-S(r)}$ $S(r)$-dimensional cylinder sets,
	denoted $D^{(j_1)}$, $j_1\in\{1,\ldots, 2^{k-S(r)} \}$.
	
	In the first step, for each $j_1$,
	the $S(r)$-dimensional cylinder set $D^{(j_1)}$ is packed by $2^{S(r-1)}$ $r$-dimensional cylinder sets $C^{(j_1),i}$, $i\in\{1,\ldots, 2^{S(r-1)} \}$.
	For each $i$, we define the star $B^{(j_1),i}$ as the radius-$1$ Hamming ball within $C^{(j_1),i}$ centered at the smallest element of $C^{(j_1),i}$ (with respect to the lexicographic order of $\{0,1\}^k$), and include it in the sequence.
	
	At this point, the sites in $D^{(j_1)}$ that have not yet been covered by stars is $D^{(j_1)} \setminus(\cup_{i} B^{(j_1),i} )$.
	This set is split into $2^r-(r+1)$ $S(r-1)$-dimensional cylinder sets,
	which we denote by $D^{(j_1,j_2)}$, $j_2\in\{1,\ldots,2^r-(r+1) \}$.
	
	Note that
	$\cup_{j_1} D^{(j_1,j_2)}$ is a cylinder set, and hence, for each $j_2$, the $(\cup_{j_1} D^{(j_1,j_2)})$-rows of a conditional distribution being processed by Algorithm~\ref{alg1} can be jointly reset by one single sharing step to achieve $p'(\cdot|x)\approx\delta_0$ for all $x\in \cup_{j_1} D^{(j_1,j_2)}$.
	
	In the second step, for each $j_2$,
	the cylinder set $D^{(j_1,j_2)}$ is packed by $2^{S(r-2)}$ $(r-1)$-dimensional cylinder sets $C^{(j_1,j_2),i}$, $i\in\{1,\ldots, 2^{S(r-2)}\}$,  and the corresponding stars are included in the sequence.
	
	The procedure is iterated until the $r$-th step. In this step, each $D^{(j_1,\ldots, j_r)}$ is a $1$-dimensional cylinder set
	and is packed by a single $1$-dimensional cylinder set $C^{(j_1,\ldots, j_r),1} = B^{(j_1,\ldots, j_r),1}$.
	Hence, at this point, all of $\{0,1\}^k$ has been exhausted and the procedure terminates.
	
	Summarizing, the procedure is initialized by creating the branches $D^{(j_1)}$, $j_1\in[2^{k-S(r)}]$.
	In the first step, each branch $D^{(j_1)}$ produces $2^{S(r-1)}$ stars and splits into the branches $D^{(j_1,j_2)}$, $j_2\in [2^r-(r+1)]$.
	More generally, in the $i$-th step, each branch $D^{(j_1,\ldots, j_i)}$ produces $2^{S(r-i)}$ stars,
	and splits into the branches $D^{(j_1,\ldots, j_i,j_{i+1})}$, $j_{i+1}\in[2^{r-(i-1)}-(r+1-(i-1))]$.
	
	The total number of stars $D^{(j_1,\ldots, j_r)}$ is given precisely by $2^{k-S(r)}$ times the value of the iterative function $F(r)=f_r(f_{r-1}(\cdots f_2(f_1)))$, whereby $f_1=1$.
	The total number of resets is given by the number of branches created from the first step on, which is precisely $R(r)=\prod_{i\in[r]}(2^i-(i+1))$.
	
	Figure~\ref{figure:stars}C offers an illustration of these star packing sequences. 
	The figure shows the case $k=S(3)=6$.
	In this case, there is only one initial branch $D^{(1)}=\{0,1\}^6$.
	The stars $B^{(1),i}$, $i\in[2^{S(2)}]=[8]$ are shown in solid blue, $B^{(1,1),i}$, $i\in[2^{S(1)}]=[2]$ in dashed red, and $B^{(1,1,1),1}$ in dotted green.
	For clarity, only these stars are highlighted.
	The stars $B^{(1,j_2),i}$ and $B^{(1,j_2,1),1}$ resulting from split branches are similar, translated versions of the highlighted ones. 
\end{proof}

With this, we obtain the general bound of the theorem:

\begin{proposition}[Theorem~\ref{theorem:universal}, general bound]
	\label{proposition:universal1}
	Let $k \geq S(r)$.
	The model $\RBM_{n,m}^k$ can approximate every conditional distribution from $\Del{k}{n}$ arbitrarily well whenever
	$m \geq m^{(r)}_{k,n}$, where $m^{(r)}_{k,n}:=2^{k- S(r)} F(r) (2^n-1) + R(r)$.
\end{proposition}

\begin{proof}
	This is in view of the complexity of Algorithm~\ref{alg1} for the sequence given in Lemma~\ref{lemma:starpacking}.
\end{proof}

In order to make the universal approximation bound more comprehensible, in Table~\ref{table:thm} we evaluated the sequence $m^{(r)}_{n,k}$ for $r=1,2,3\ldots$ and $k\geq S(r)$.
Furthermore, the next proposition gives an explicit expression for the coefficients $2^{-S(r)}F(r)$ and $R(r)$ appearing in the bound.
This yields the second part of Theorem~\ref{theorem:universal}.
In general, the bound $m^{(r)}_{n,k}$ decreases with increasing $r$, except possibly for a few values of $k$ when $n$ is small.
For a pair $(k,n)$, any $m^{(r)}_{n,k}$ with $k\geq S(r)$ is a sufficient number of hidden units for obtaining a universal approximator.

\begin{table}
	\centering
	\begin{tabular}{c | l c l r c c r }
		&  \multicolumn{6}{l}{$m_{n,k}^{(r)} = $}   \\[4pt]
		$r$&   $2^{k}$ &$2^{-S(r)} $ &   $F(r)$ & $(2^n-1)$& $ +$& $R(r)$   \\ 
		\hline
		$1$  	 &$2^{k}$ &$2^{-1} $ &   $1$ & $(2^n-1)$& $ +$&    $0$   \\
		$2$ 	 &$2^{k}$ &$2^{-3} $ &   $3$ & $(2^n-1)$& $ +$&    $1$   \\
		$3$ 	 &$2^{k}$ &$2^{-6} $ &  $20$ & $(2^n-1)$& $ +$&    $4$   \\
		$4$ 	 &$2^{k}$ &$2^{-10}$ & $284$ & $(2^n-1)$& $ +$&   $44$   \\
		$5$ 	 &$2^{k}$ &$2^{-15}$ &$8408$ & $(2^n-1)$& $ +$& $1144$   \\
		$\vdots$&$\vdots$&\multicolumn{2}{c}{$\vdots$} 	&$\vdots$& $\vdots$&$\vdots$  \\
		$>17$  &$2^{k}$ &\multicolumn{2}{c}{$0.2263$} &$(2^n-1)$& $+$ & $ 2^{S(r)} 0.0269$
	\end{tabular}
	\caption{Numerical evaluation of the bounds from Proposition~\ref{proposition:universal1}.
		Each row evaluates the universal approximation bound $m^{(r)}_{n,k}$ for a value of $r$.
	}
	\label{table:thm}
\end{table}

\begin{proposition}[Theorem~\ref{theorem:universal}, explicit bounds]
	\label{proposition:secondp}
	The function $K(r):=2^{-S(r)}F(r)$ is bounded from below and above as $K(6) \prod_{i=7}^r\left( 1 -\frac{i-3}{2^i}\right)\leq  K(r) \leq K(6) \prod_{i=7}^r\left( 1 -\frac{i-4}{2^i}\right)$ for all $r\geq 6$. Furthermore, $K(6)\approx 0.2442$ and $K(\infty)\approx 0.2263$.
	Moreover, $R(r):=\prod_{i=2}^r (2^i -(i+1)) = 2^{S(r)} P(r)$, where $P(r):=\frac{1}{2}\prod_{i=2}^r (1 - \frac{(i+1)}{2^i})$, and
	$P(\infty) \approx 0.0269$. 
\end{proposition}

\begin{proof}
	From the definition of $S(r)$ and $F(r)$, we obtain that
	\begin{equation}
	K(r) = 2^{-r} + K(r-1)(1 - 2^{-r}(r+1)).
	\label{eq:K}
	\end{equation}
	Note that $K(1)=\frac{1}{2}$, and that $K(r)$ decreases monotonically.
	
	Now, note that if $K(r-1)\leq \frac{1}{c}$, then the left hand side of Equation~\eqref{eq:K} is bounded from below as
	$K(r) \geq K(r-1)(1 - 2^{-r}(r+1-c))$.
	For a given $c$, let $r^c$ be the first $r$ for which $K(r-1)\leq \frac{1}{c}$, assuming that such an $r$ exists.
	Then
	\begin{equation}
	K(r)\geq K(r^c-1)\prod_{i=r^c}^r \left( 1 -\frac{i+1-c}{2^i}\right),\quad \text{for all $r\geq r^c$}.
	\label{eq:Klower}
	\end{equation}
	Similarly, if $K(r)>\frac{1}{d}$ for all $r\geq r^{b}$, then
	\begin{equation*}
	K(r) \leq K(r^b-1) \prod_{i=r^b}^r \left( 1 -\frac{i+1-b}{2^i}\right),\quad\text{for any $r\geq r^b$}.
	\end{equation*}
	Direct computations show that $K(6)\approx 0.2445 \leq \frac{1}{4}$.
	On the other hand, using the computational engine \texttt{Wolfram|Alpha(access June 01, 2014)} we obtain that $\prod_{i=0}^\infty \left( 1 -\frac{i-3}{2^i}\right)\approx 7.7413$.
	Plugging both terms into Equation~\eqref{eq:Klower} yields that
	$K(r)$ is always bounded from below by $0.2259$.
	
	Since $K(r)$ is never smaller than or equal to $\frac{1}{5}$,
	we obtain that $K(r) \leq K(r'-1) \prod_{i=r'}^r \left( 1 -\frac{i-4}{2^i}\right)$, for any $r'$ and $r\geq r'$.
	Using $r'=7$, the right hand side evaluates in the limit of large $r$ to approximately $0.2293$.
	
	Numerical evaluation of $K(r)$ from Equation~\eqref{eq:K} for $r$ up to one million (using \texttt{Matlab R2013b}) indicates that, indeed,
	$K(r)$ tends to approximately $0.2263$ for large $r$.
\end{proof}

We close this subsection with the remark that the proof strategy can be used not only to study universal approximation, but also approximability of selected classes of conditional distributions: 

\begin{remark}
	\label{remark:supp}
	\normalfont
	If we only want to model a restricted class of conditional distributions,
	then adapting Algorithm~\ref{alg1} to these restrictions may yield tighter bounds for the number of hidden units that suffices to represent these restricted conditionals. For example:
	
	If we only want to model the target conditionals $q(\cdot|x)$ for the inputs $x$ from a subset $\Scal \subseteq \{0,1\}^k$ and do not care about $q(\cdot|x)$ for $x\not\in\Scal$, then in the algorithm we just need to replace $\{0,1\}^k$ by $\Scal$. In this case, a cylinder set packing of $\Scal\setminus\Bcal$ is understood as a collection of disjoint cylinder sets $C^1,\ldots, C^K\subseteq \{0,1\}^k$ with $\cup_{i\in[K]}C^i\supseteq \Scal\setminus\Bcal$ and $(\cup_{i\in[K]}C^i) \cap \Bcal=\emptyset$.
	
	Furthermore, if for some cylinder set $C^i$ and a corresponding star $B^i\subseteq C^i$ the conditionals $q(\cdot|x)$ with $x\in B^i$ have a common support set $T\subseteq\{0,1\}^n$, then the $C^i$-rows of $p$ can be reset to a distribution $\delta_y$ with $y\in T$, and only $|T|-1$ sharing steps are needed to transform $p$ to a distribution whose conditionals approximate $q(\cdot|x)$ for all $x\in B^i$ to any desired accuracy.
	In particular, for the class of target conditional distributions with $\supp q(\cdot|x) = T$ for all $x$,
	the term $2^n-1$ in the complexity bound of Algorithm~\ref{alg1} is replaced by $|T|-1$.
\end{remark}

\subsection{Necessary Number of Hidden Units}
\label{subsec:bounds}

Proposition~\ref{proposition:universallower} follows from simple parameter counting arguments.
In order to make this rigorous, first we make the observation that universal approximation of (conditional) probability distributions by Boltzmann machines or any other models based on exponential families, with or without hidden variables, requires the number of model parameters to be as large as the dimension of the set being approximated.
We denote by $\Del{\Xcal}{\Ycal}$ the set of conditionals with inputs form a finite set $\Xcal$ and outputs from a finite set $\Ycal$.
Accordingly, we denote by $\Delta_\Ycal$ the set of probability distributions on $\Ycal$. 

\begin{lemma}\label{lemma:parametercounting}
	Let $\Xcal$, $\Ycal$, and $\Zcal$ be some finite sets. 
	Let $\Mcal\subseteq\Del{\Xcal}{\Ycal}$ be defined as the set of conditionals of the marginal  $\Mcal'\subseteq\Delta_{\Xcal\times\Ycal}$ of an exponential family $\Ecal \subseteq\Delta_{ \Xcal\times\Ycal\times\Zcal}$. 
	If $\Mcal$ is a universal approximator of conditionals from $\Del{\Xcal}{\Ycal}$, 
	then $\dim(\Ecal) \geq \dim(\Del{\Xcal}{\Ycal})=|\Xcal|(|\Ycal|-1)$.
\end{lemma} 

The intuition of this lemma is that, for models defined by marginals of exponential families, the set of conditionals that can be approximated arbitrarily well is essentially equal to the set of conditionals that can be represented exactly, implying that there are no low-dimensional universal approximators of this type.

\begin{proof}[of Lemma~\ref{lemma:parametercounting}]
	We consider first the case of probability distributions; that is, the case with $|\Xcal| =1$ and $\Xcal\times \Ycal \cong \Ycal$.
	Let $\Mcal$ be the image of the exponential family $\Ecal$ by a differentiable map $f$ (for example, the marginal map).
	The closure $\overline{\Ecal}$, which consists of all distributions that can be approximated arbitrarily well by $\Ecal$, is a compact set.
	Since $f$ is continuous, the image of $\overline{\Ecal}$ is also compact, and $\overline{\Mcal}=\overline{f(\Ecal)}=f(\overline{\Ecal})$.
	The model $\Mcal$ is a universal approximator if and only if $\overline{\Mcal}=\Delta_{\Ycal}$.
	The set $\overline{\Ecal}$ is a finite union of exponential families;
	one exponential family $\Ecal_F$ for each possible support set $F$ of distributions from $\overline{\Ecal}$.
	When $\dim(\Ecal) < \dim(\Delta_{\Ycal})$, each point of each $\Ecal_F$ is a critical point of $f$ (the Jacobian is not surjective at that point).
	By Sard's theorem, each $\Ecal_F$ is mapped by $f$ to a set of measure zero in $\Delta_\Ycal$.
	Hence the finite union $\cup_F f(\Ecal_F) = f(\cup_F \Ecal_F)=f(\overline{\Ecal})=\overline{\Mcal}$ has measure zero in $\Delta_{\Ycal}$.
	
	For the general case, with $|\Xcal|\geq 1$, note that $\Mcal \subseteq\Del{\Xcal}{\Ycal}$ is a universal approximator iff the joint model $\Delta_{\Xcal}\Mcal=\{p(x)q(y|x)\colon p\in\Delta_\Xcal, q\in\Mcal \} \subseteq\Delta_{\Xcal\times\Ycal}$ is a universal approximator.
	The latter is the marginal of the exponential family $\Delta_{\Xcal}\ast\Ecal=\{ p\ast q\colon p\in\Delta_\Xcal, q\in\Ecal \}\subseteq\Delta_{\Xcal\times\Ycal\times\Zcal}$.
	Hence the claim follows from the first part.
\end{proof}

\begin{proof}[of Proposition~\ref{proposition:universallower}]
	If $\RBM_{n,m}^k$ is a universal approximator of conditionals from $\Del{k}{n}$,
	then the model consisting of all probability distributions of the form $p(x,y) = \frac{1}{Z}\sum_z\exp(z^\top W y + z^\top V x + b^\top y + c^\top z + f(x))$ is a universal approximator of probability distributions from $\Delta_{k+n}$.
	The latter is the marginal of an exponential family of dimension $mn + mk + n + m + 2^k-1$.
	Thus, by Lemma~\ref{lemma:parametercounting}, $m \geq \frac{2^{k+n}-2^k-n}{(n + k + 1)}$.
\end{proof}

\section{Details on the Maximal Approximation Errors}

\begin{proof}[of Proposition~\ref{proposition:errors}]
	We have that $D_{\RBM_{n,m}^k} \leq \max_{p\in\Delta_{k+n}\colon p_X=u_X}D(p\| \RBM_{n+k,m} )$.
	The right hand side is bounded by $n$, since the RBM model contains the uniform distribution.
	It is also bounded by the maximal divergence $D_{\RBM_{n+k,m}}\leq (n+k)-\lfloor\log_2(m+1)\rfloor -\frac{m+1}{2^{\lfloor\log_2(m+1)\rfloor}}$~\citep[][]{GSI2013}.
\end{proof}

In order to prove Theorem~\ref{theorem:errors}, we will upper bound the approximation errors of CRBMs by the approximation errors of submodels of CRBMs. First, we note the following:

\begin{lemma}
	\label{lemma:divergprod}
	The maximal divergence of a conditional model that is a Cartesian product of a probability model is bounded from above by the maximal divergence of that probability model:
	if $\Mcal = \times_{x\in\{0,1\}^k}\Ncal\subseteq\Del{k}{n}$ for some $\Ncal\subseteq\Delta_n$, then
	$D_{\Mcal}\leq D_{\Ncal}$.
\end{lemma}

\begin{proof}
	For any $p\in\Del{k}{n}$, we have
	\begin{eqnarray*}
		D(p\|\Mcal) 
		&=& \inf_{q\in\Mcal}\frac{1}{2^k}\sum_x D(p(\cdot|x) \| q(\cdot|x))\\
		&=& \frac{1}{2^k}\sum_x \inf_{q(\cdot|x)\in\Ncal}D(p(\cdot|x) \| q(\cdot|x))\\
		&\leq& \frac{1}{2^k} \sum_{x} D_{\Ncal} =  D_\Ncal. 
	\end{eqnarray*}
	\vspace{-.5cm}
\end{proof}

\begin{definition}
	Given a partition $\Zcal =\{\Ycal_1,\ldots, \Ycal_L \}$ of $\{0,1\}^n$, the \emph{partition model} $\Pcal_\Zcal\subseteq\Delta_n$ is the set of all probability distributions on $\{0,1\}^n$ with constant value on each partition block. 
\end{definition}

The set $\{0,1\}^l$, $l\leq n$ naturally defines a partition of $\{0,1\}^n$ into cylinder sets $\{y\in\{0,1\}^n\colon y_{[l]} = z \}$ for all $z\in\{0,1\}^l$.  The divergence from $\Pcal_{\Zcal}$ is bounded from above by $D_{\Pcal_{\Zcal}}\leq l-n$.

Now, the model $\RBM_{n,m}^k$ can approximate certain products of partition models arbitrarily well:

\begin{proposition}
	\label{theorem:universalpartition}
	Let $\Zcal=\{0,1\}^l$ with $l\leq n$. 
	Let $r$ be any integer with $k\geq S(r)$.
	The model $\RBM_{n,m}^k$ can approximate any conditional distribution from the product of partition models $\Pcal_\Zcal^k:=\Pcal_\Zcal\times\cdots\times\Pcal_\Zcal$ arbitrarily well whenever $m\geq 2^{k-S(r)}F(r)(|\Zcal|-1)+R(r)$.
\end{proposition}

\begin{proof}
	This is analogous to the proof of Proposition~\ref{theorem:universalsubset}, with a few differences.
	Each element $z$ of $\Zcal$ corresponds to a cylinder set $\{y\in\{0,1\}^n\colon y_{[l]} =z \}$ and the collection of cylinder sets for all $z\in\Zcal$ is a partition of $\{0,1\}^n$. Now we can run Algorithm~\ref{alg1} in a slightly different way, with sharing steps defined by $p'=\lambda p + (1-\lambda) u_{z}$, where $u_{z}$ is the uniform distribution on the cylinder set corresponding to $z$.
\end{proof}

\begin{proof}[of Theorem~\ref{theorem:errors}]
	This follows directly from Lemma~\ref{lemma:divergprod} and Proposition~\ref{theorem:universalpartition}. 
\end{proof}

\section{Details on the Representation of Conditional Distributions from Markov Random Fields}

The proof of Theorem~\ref{theorem:interaction} is based on ideas from~\citet[][]{Younes1996109}, who discussed the universal approximation property of Boltzmann machines. 
We will use the following~\citep[][Lemma~1]{Younes1996109}: 

\begin{lemma}
	\label{lemma:Younes}
	Let $\varrho$ be a real number. Consider a fixed integer $N$ and binary variables $x_1,\ldots, x_N$. 
	There are real numbers $w$ and $b$ such that:
	\begin{itemize}
		\item If $\varrho\geq0$, 
		$\log\left( 1 + \exp(w(x_1 + \cdots + x_N) + b) \right) = \varrho \prod_i x_i + Q(x_1,\ldots, x_N)$.   
		\item If $\varrho\leq0$,
		$\log\left( 1 + \exp(w(x_1 + \cdots + x_{N-1} - x_N) + b) \right) = \varrho \prod_i x_i + Q(x_1,\ldots, x_N)$. 
	\end{itemize}
	Where $Q$ is in each case a polynomial of degree less than $N-1$ in $x_1,\ldots, x_N$. 
\end{lemma}

The following is a generalization of~\citep[][Lemma~2]{Younes1996109}: 

\begin{lemma}
	\label{lemma:interaction}
	Let $I$ and $J$ be two simplicial complexes on $[n]$ with $J\subseteq I$. 
	If $p$ is any distribution from $\Ecal_I$ and $m \geq |\{A\in I\setminus J \colon |A|>1 \}|$, 
	then there is a distribution $p'\in \Ecal_J$, such that $p\ast p'$ is contained in $\RBM_{n,m}$. 
\end{lemma}

\begin{proof}
	The proof follows closely the arguments presented in~\citep[][Lemma~2]{Younes1996109}.  
	Let $K = \{A\in I\setminus J \colon |A|>1 \}$. 
	Consider an RBM with $n$ visible units and $m = |K|$ hidden units. 
	Consider a joint distribution $q(x,u) = \frac{1}{Z} \exp(H(x,u))$ of the fully observable RBM, defined as follows. 
	We label the hidden units by subsets $A \in K$. 
	For each $A\in K$, let $s(A)$ denote the largest element of $A$, and let 
	\begin{align*}
	H(x,u) =& \sum_{A\in K} u_A  \left( w_A  S_A^{\epsilon_A}(x_A) + b_A \right) + \sum_{s\in [n]} b_s x_s , 
	\intertext{where}
	S_A^{\epsilon_A}(x_A) =& \Big(\sum_{s\in A, s<s(A)}x_s\Big) + \epsilon_A x_{s(A)}, 
	\end{align*}
	for some $\epsilon_A\in\{-1, +1\}$, $w_A, b_A , b_s\in \R$ that we will specify further below. 
	
	Denote the log probabilities of $p(x)$ and $p'(x)$ by  
	\begin{equation*}
	E(x) = \sum_{A\in I} \theta_A \prod_{i\in A} x_i \quad \text{and} \quad E'(x) = \sum_{A\in J} \vartheta_A \prod_{i\in A} x_i . 
	\end{equation*} 
	
	We obtain the desired equality $(p\ast p' )(x) = \sum_u q(x,u)$ when 
	\begin{equation}
	E(x) = \log\left(\sum_u \exp(H(x,u))\right)  - \sum_{A\in J} \vartheta_A \prod_{i\in A} x_i , 
	\label{eq:conditioninter}
	\end{equation}
	for some choice of $\vartheta_A$, for $A\in J$, some choice of $\epsilon_A, w_A, b_A$, for $A\in K$, and some choice of $b_s$, for $s\in[n]$. 
	We have 
	\begin{eqnarray*}
		\log\left(\sum_u \exp(H(x,u))\right)
		&=& \log\left( \sum_u \exp\Big( \sum_A u_A (w_A S_A^{\epsilon_A}(x_A) + b_A)   + \sum_{s\in [n]} b_s x_s\Big) \right)\\ 
		&=& \log\left( 
		\Big(\sum_u \prod_A \exp( u_A ( w_A S_A^{\epsilon_A}(x_A) + b_A)) \Big) \exp\Big(\sum_{s\in [n]} b_s x_s\Big) \right)\\
		&=& \log\left( 
		\Big(\prod_A \sum_{u_A} \exp( u_A (w_A S_A^{\epsilon_A}(x_A) + b_A)) \Big) \exp\Big(\sum_{s\in [n]} b_s x_s\Big) \right)\\
		&=& \sum_A\log(1 + \exp(w_A S_A^{\epsilon_A}(x_A) + b_A)) + \sum_{s\in [n]} b_s x_s 
	\end{eqnarray*}
	
	The terms 
	\begin{equation*}
	\phi_A^{\epsilon_A} (x_A) := \log \left( 1 + \exp(w_A S_A^{\epsilon_A}(x_A) + b_A) \right) 
	\end{equation*}
	are of the same form as the functions from Lemma~\ref{lemma:Younes}. 
	
	To solve Equation~\eqref{eq:conditioninter}, 
	we first apply Lemma~\ref{lemma:Younes} 
	on $\phi_{A}^{\epsilon_{A}}$ to cancel the terms 
	$\theta_A \prod_{i\in A} x_i$ of $E(x)$, 
	for which $A$ is a maximal element of $I\setminus J$ of cardinality more than one. 
	This involves choosing appropriate $\epsilon_{A}\in\{-1, +1\}$, $w_{A}$ and $b_{A}$, for the corresponding $A$. 
	The remaining polynomial consists of terms with strictly smaller monomials. 
	We apply lemma~\ref{lemma:Younes} repeatedly on this polynomial, until only 
	monomials with $A\in J$ or $|A|=1$ remain. 
	These terms are canceled with $\vartheta_A \prod_{i\in A}x_i$, $A\in J$, or with $b_s x_s$, $s\in[n]$. 
\end{proof}

\begin{proof}[of Theorem~\ref{theorem:interaction}]
	By Lemma~\ref{lemma:interaction}, 
	there is a $p'\in \Ecal_J$, $J=2^{[k]}$, such that $p\ast p'$ is in $\RBM_{k+n,m}$. 
	Now, the conditionals distribution $(p\ast p')(y|x)$ of the last $n$ units, given the first $k$ units, are idependent of $p'$, since this is independent of $y$. 
	\end{proof}

\begin{proof}[of Corollary~\ref{corollary:interaction}]
	The statement follows from Theorem~\ref{theorem:interaction}, considering the simplicial complex $I = 2^{[k]}\times J$ and a joint probability distribution $p\in \Ecal_I\subseteq\Delta_{k+n}$ with  the desired conditionals $p(\cdot|x)=p^x$. 
\end{proof}

\section{Details on the Approximation of Conditional Distributions with Restricted Supports}

\begin{proof}[of Proposition~\ref{proposition:boundedsupport}]
	This follows from the fact that $\RBM_{n+k,m}$ can approximate any probability distribution with support of cardinality $m+1$ arbitrarily well~\citep{Montufar2011}.
\end{proof}

\begin{proof}[of Proposition~\ref{theorem:universalsubset}]
	This is analogous to the proof of Proposition~\ref{proposition:universal1}.
	The complexity of Algorithm~\ref{alg1} as evaluated there does not depend on the specific structure of the support sets, but only on their cardinality, as long as they are the same for all $x$.
\end{proof}

The following lemma states that a  CRBM can compute all deterministic conditionals that can be computed by a feedforward linear threshold network with the same number of hidden units. 

\begin{lemma}
	\label{lemma:feedforward}
	Consider a function $f\colon \{0,1\}^k\to\{0,1\}^n$. 
	The model $\RBM^k_{n,m}$ can approximate the deterministic policy $p(y|x) = \delta_{f(x)}(y)$ arbitrarily well, whenever  
	this can be represented by a feedforward linear threshold network with $m$ hidden units; that is, when 
	\begin{equation*}
	f(x) = \hs(W^\top (\hs( V x + c) ) + b), \quad\text{for all $x\in\{0,1\}^k$}, 
	\end{equation*} 
	for some generic choice of $W,V,b,c$. 
\end{lemma}

\begin{proof} 
	Consider the conditional distribution $p(\cdot|x)$. 
	This is the visible marginal of $p(y,z|x) = \frac{1}{Z} \exp((V x +c)^\top z + b^\top y + z^\top W y )$. 
	Consider weights $\alpha$ and $\beta$, 
	with $\alpha$ large enough, such that $\argmax_z(\alpha V x + \alpha c)^\top z = \argmax_z (\alpha V x + \alpha c)^\top z + (\beta W^\top z + \beta b)^\top y$ for all $y\in\{0,1\}^n$. 
	Note that for generic choices of $V$ and $c$, the set $\argmax_z(\alpha V + \alpha c)^\top z$ consists of a single point $z^\ast = \hs(V x + c)$. 
	We have $\argmax_{(y,z)} (\alpha V x + \alpha c)^\top z + (\beta W^\top z + \beta b)^\top y = (z^\ast, \argmax_y (\beta W^\top z^\ast + \beta b)^\top y)$. 
	Here, again, for generic choices of $V$ and $b$, the set $\argmax_y(\beta W^\top z^\ast +\beta b)^\top y$ consists of a single point $y^\ast=\hs(W^\top z^\ast + b)$. 
	The joint distribution $p(y,z|x)$ with parameters $t \beta W, t\alpha V,  t\beta b, t\alpha c$ tends to the point measure $\delta_{(y^\ast, z^\ast)}(y,z)$ as $t\to\infty$. 
	In this case $p(y|x)$ tends to $\delta_{y^\ast}(y)$ as $t\to\infty$, 
	where $y^\ast = \hs(W^\top z^\ast + b) = \hs(W^\top \hs(V x + c) +b)$, for all $x\in\{0,1\}^k$. 
\end{proof}

\begin{proof}[of Theorem~\ref{theorem:ltn}]
	The second statement is precisely Lemma~\ref{lemma:feedforward}. 
	For the more general statement the arguments are as follows. 
	Note that the conditional distribution $p(y|z)$ of the output units, given the hidden units, 
	is the same for a CRBM and for its feedforward network version. 
	Furthermore, for each input $x$, the CRBM output distribution is $p(y|x) = \sum_z  (q(z|x)\ast p(z) )  p(y|z)$, where 
	\begin{equation*}
	q(z|x) = \frac{\exp(z^\top V x + c^\top z)}{\sum_{z'} \exp(z'^\top V x + c^\top z')}
	\end{equation*}
	is the conditional distribution represented by the first layer, 
	\begin{equation*}
	p(y,z) = \frac{\exp(z^\top W y + b^\top y)}{\sum_{y',z'}\exp(z'^\top W y' + b^\top y')}
	\end{equation*}
	is the distribution represented by the RBM with parameters $W,b,0$, and 
	\begin{equation*}
	q(z|x)\ast p(z)  = \frac{q(z|x) p(z) }{ \sum_{z'}q(z'|x) p(z')},\quad\text{for all $z$}
	\end{equation*} 
	is the renormalized entry-wise product of the conditioned distribution $q(\cdot|x)$ and the RBM hidden marginal distribution 
	\begin{equation*}
	p(z) = \sum_{y} p(y,z). 
	\end{equation*}
	Now, if $q$ is deterministic, 
	then $q(z|x)\ast p(z)$ is the same as $q(z|x)$, regardless of $p(z)$ (strictly positive). 
\end{proof}

The proof of Theorem~\ref{proposition:universaldeterministic} builds on the following lemma, which describes a combinatorial property of the deterministic policies that can be approximated arbitrarily well by CRBMs. 
Recall that the Heaviside step function $\hs$ maps a real number $a$ to $0$ if $a<0$, to $1/2$ if $a=0$, and to $1$ if $a>0$.

\begin{lemma}
	\label{lemma:deter}
	Consider a function $f\colon\{0,1\}^k\to\{0,1\}^n$.
	The model $\RBM_{n,m}^k$ can approximate the deterministic policy $p(y|x) = \delta_{f(x)}(y)$ arbitrarily well
	only if there is a choice of the model parameters $W,V,b,c$ for which
	\begin{equation*}
	f(x)=\hs(W^\top \hs([W,V] \left[\begin{smallmatrix}f(x) \\x\end{smallmatrix}\right] +c) +b),\quad\text{for all $x\in\{0,1\}^k$},
	\label{eq:deter}
	\end{equation*}
	where the Heaviside function $\hs$ is applied entry-wise to its argument.
\end{lemma}

\begin{proof} 
	Consider a choice of $W,V,b,c$.
	For each input state $x$,
	the conditional represented by $\RBM_{n,m}^k$ is equal to the mixture distribution
	$p(y|x) = \sum_z p(z|x)p(y|x,z)$,
	with mixture components $p(y|x,z) = p(y|z) \propto \exp((z^\top W +b^\top) y)$ and mixture weights $p(z|x) \propto \sum_{y'} \exp((z^\top W + b^\top)y' + z^\top( V x +c))$ for all $z\in\{0,1\}^m$.
	The support of a mixture distribution is equal to the union of the supports of the mixture components with non-zero mixture weights.
	In the present case, 
	if $\sum_y|p(y|x)-\delta_{f(x)}(y)|\leq \alpha$, then $\sum_y|p(y|x,z)-\delta_{f(x)}(y)|\leq \alpha/\epsilon$ for all $z$ with $p(z|x)>\epsilon$, for any $\epsilon>0$.
	Choosing $\alpha$ small enough, $\alpha/\epsilon$ can be made arbitrarily small for any fixed $\epsilon>0$. In this case, for every $z$ with $p(z|x)>\epsilon$, necessarily
	\begin{equation}
	(z^\top W + b^\top)f(x)\gg (z^\top W + b^\top)y, \quad\text{for all $y\neq f(x)$},
	\label{eq:deterb}
	\end{equation}
	and hence
	\begin{equation*}
	\sgn (z^\top W +b^\top) = \sgn(f(x)-\tfrac{1}{2}).
	\end{equation*}
	
	Furthermore, the probability assigned by $p(z|x)$ to all $z$ that do not satisfy Equation~\eqref{eq:deterb} has to be very close to zero (upper bounded by a function that decreases with $\alpha$).
	The probability of $z$ given $x$ is given by
	\begin{equation*}
	p(z|x)
	= \frac{1}{Z_{z|x}} \exp(z^\top( V x +c)) \sum_{y'}\exp((z^\top W + b^\top)y').
	\end{equation*}
	In view of Equation~\eqref{eq:deterb}, for all $z$ with $p(z|x)>\epsilon$, if $\alpha$ is small enough,  $p(z|x)$ is arbitrarily close to
	\begin{equation*}
	\frac{1}{Z_{z|x}} \exp( z^\top(V x +c)) \exp((z^\top W + b^\top)f(x) ).
	\label{eq:unaeqs}
	\end{equation*}
	This holds, in particular, for every $z$ that maximizes $p(z|x)$. Therefore,  
	\begin{equation*}
	\argmax_z p(z|x)
	=  \argmax_z z^\top( W f(x) + V x +c). 
	\end{equation*}
	Each of these $z$ must satisfy Equation~\eqref{eq:deterb}. 
	This completes the proof.
\end{proof}

\begin{proof}[of Theorem~\ref{proposition:universaldeterministic}]
	\emph{Sufficient condition: }
	The bound $2^k-1$ follows directly from Proposition~\ref{proposition:boundedsupport}. 
	For the second bound, note that any function $f\colon \{0,1\}^k \to \{0,1\}^n;\; x\mapsto y$ can be computed by a parallel composition of the functions $f_i\colon x\mapsto y_i$, for all $i\in[n]$. 
	Hence the bound follows from Lemma~\ref{lemma:feedforward} and the fact that a feedforward linear threshold network with $\frac{3}{k+2}2^k$ hidden units can compute any Boolean function. 
	
	\emph{Necessary condition: }
	Recall that a linear threshold function with $N$ input bits and $M$ output bits is a function of the form
	$\{0,1\}^N\to\{0,1\}^M$; $y\mapsto \hs(W y + b)$ with $W\in\R^{M\times N}$ and $b\in\R^M$.
	Lemma~\ref{lemma:deter} shows that each deterministic policy that can be approximated by $\RBM_{n,m}^k$ arbitrarily well corresponds to the $y$-coordinate fixed points of a map defined as the composition of two linear threshold functions
	$\{0,1\}^{k+n}\to \{0,1\}^m$; $(x,y)\mapsto \hs([W,V] \left[\begin{smallmatrix}y \\x\end{smallmatrix}\right] +c)$
	and $\{0,1\}^m \to \{0,1\}^n$; $z\mapsto \hs(W^\top z +b)$.
	In particular, we can upper bound the number of deterministic policies that can be approximated arbitrarily well by $\RBM_{n,m}^k$, by the total number of compositions of two linear threshold functions; one with $n+k$ inputs and $m$ outputs and the other with $m$ inputs and $n$ outputs. 
	
	Let $\LTF(N,M)$ be the number of linear threshold functions with $N$ inputs and $M$ outputs.
	It is known that~\citep{857765, Wenzel:2000:HAS:361159.361171}
	\begin{equation*}
	\LTF(N,M) \leq 2^{N^2 M}.
	\end{equation*}
	The number of deterministic policies that can be approximated arbitrarily well by $\RBM_{n,m}^k$ is thus bounded above by
	$\LTF(n+k,m)\cdot \LTF(m,n) \leq 2^{m(n+k)^2 + n m^2}$.
	The actual number may be much smaller, in view of the fixed-point and shared parameter constraints.
	On the other hand, the number of deterministic policies in $\Del{k}{n}$ is as large as $(2^n)^{2^k}=2^{n 2^k}$.
	The claim follows from comparing these two numbers.
\end{proof}

\vskip .2in

\subsection*{Acknowledgment}

We acknowledge support from the DFG Priority Program Autonomous Learning (DFG-SPP 1527).
G.~M. and K.~G.-Z. would like to thank the Santa Fe Institute for hosting them during the initial work on this article. 

\renewcommand\bibsection{\subsection*{\refname}}
\bibliography{referenzen}

\begin{thebibliography}{32}
\providecommand{\natexlab}[1]{#1}
\providecommand{\url}[1]{\texttt{#1}}
\expandafter\ifx\csname urlstyle\endcsname\relax
  \providecommand{\doi}[1]{doi: #1}\else
  \providecommand{\doi}{doi: \begingroup \urlstyle{rm}\Url}\fi

\bibitem[Ay et~al.(2013)Ay, Mont\'ufar, and Rauh]{ICCN}
N.~Ay, G.~Mont\'ufar, and J.~Rauh.
\newblock Selection criteria for neuromanifolds of stochastic dynamics.
\newblock In Y.~Yamaguchi, editor, \emph{Advances in Cognitive Neurodynamics
  (III)}, pages 147--154. Springer, 2013.
\newblock URL \url{http://dx.doi.org/10.1007/978-94-007-4792-0_20}.

\bibitem[Bellman(1957)]{Bellman:1957}
R.~E. Bellman.
\newblock \emph{Dynamic programming}.
\newblock Princeton University Press, Princeton, NY, 1957.

\bibitem[Bengio(2009)]{Bengio:2009:LDA:1658423.1658424}
Y.~Bengio.
\newblock Learning deep architectures for {AI}.
\newblock \emph{Found. Trends Mach. Learn.}, 2\penalty0 (1):\penalty0 1--127,
  Jan. 2009.
\newblock URL \url{http://dx.doi.org/10.1561/2200000006}.

\bibitem[Cueto et~al.(2010)Cueto, Morton, and Sturmfels]{Cueto2010}
M.~A. Cueto, J.~Morton, and B.~Sturmfels.
\newblock Geometry of the restricted {B}oltzmann machine.
\newblock In M.~Viana and H.~Wynn, editors, \emph{Algebraic methods in
  statistics and probability II, AMS Special Session}, volume~2. AMS, 2010.

\bibitem[Fischer and Igel(2012)]{FisherAnIntroduction}
A.~Fischer and C.~Igel.
\newblock An introduction to restricted {B}oltzmann machines.
\newblock In L.~Alvarez, M.~Mejail, L.~Gomez, and J.~Jacobo, editors,
  \emph{Progress in Pattern Recognition, Image Analysis, Computer Vision, and
  Applications}, volume 7441 of \emph{Lecture Notes in Computer Science}, pages
  14--36. Springer Berlin Heidelberg, 2012.
\newblock URL \url{http://dx.doi.org/10.1007/978-3-642-33275-3_2}.

\bibitem[Freund and Haussler(1994)]{freund1994unsupervised}
Y.~Freund and D.~Haussler.
\newblock \emph{Unsupervised Learning of Distributions of Binary Vectors Using
  Two Layer Networks}.
\newblock Technical report. Computer Research Laboratory, University of
  California, Santa Cruz, 1994.

\bibitem[Gilbert(1952)]{Gilbert:1952}
E.~N. Gilbert.
\newblock A comparison of signalling alphabets.
\newblock \emph{Bell System Technical Journal}, 31:\penalty0 504--522, 1952.

\bibitem[Hinton(2002)]{Hinton:2002:TPE:639729.639730}
G.~E. Hinton.
\newblock Training products of experts by minimizing contrastive divergence.
\newblock \emph{Neural Computation}, 14\penalty0 (8):\penalty0 1771--1800,
  2002.
\newblock URL \url{http://dx.doi.org/10.1162/089976602760128018}.

\bibitem[Hinton(2012)]{HintonAParactical}
G.~E. Hinton.
\newblock A practical guide to training restricted boltzmann machines.
\newblock In G.~Montavon, G.~B. Orr, and K.-R. M\"uller, editors, \emph{Neural
  Networks: Tricks of the Trade}, volume 7700 of \emph{Lecture Notes in
  Computer Science}, pages 599--619. Springer Berlin Heidelberg, 2012.
\newblock URL \url{http://dx.doi.org/10.1007/978-3-642-35289-8_32}.

\bibitem[Hinton et~al.(2006)Hinton, Osindero, and Teh]{hinton2006fast}
G.~E. Hinton, S.~Osindero, and Y.-W. Teh.
\newblock A fast learning algorithm for deep belief nets.
\newblock \emph{Neural Computation}, 18\penalty0 (7):\penalty0 1527--1554,
  2006.

\bibitem[Larochelle and Bengio(2008)]{LarochelleB08}
H.~Larochelle and Y.~Bengio.
\newblock Classification using discriminative restricted {B}oltzmann machines.
\newblock In W.~W. Cohen, A.~McCallum, and S.~T. Roweis, editors,
  \emph{Proceedings of the 25th International Conference on Machine Learning
  (ICML 2008)}, pages 536--543. ACM, 2008.

\bibitem[Le~Roux and Bengio(2008)]{LeRoux:2008:RPR:1374176.1374187}
N.~Le~Roux and Y.~Bengio.
\newblock Representational power of restricted {B}oltzmann machines and deep
  belief networks.
\newblock \emph{Neural Computation}, 20\penalty0 (6):\penalty0 1631--1649,
  2008.

\bibitem[Martens et~al.(2013)Martens, Chattopadhya, Pitassi, and
  Zemel]{NIPS2013_5020}
J.~Martens, A.~Chattopadhya, T.~Pitassi, and R.~Zemel.
\newblock On the expressive power of restricted {B}oltzmann machines.
\newblock In C.~Burges, L.~Bottou, M.~Welling, Z.~Ghahramani, and
  K.~Weinberger, editors, \emph{Advances in Neural Information Processing
  Systems 26}, pages 2877--2885. Curran Associates, Inc., 2013.
\newblock URL
  \url{http://papers.nips.cc/paper/5020-on-the-expressive-power-of-restricted-boltzmann-machines.pdf}.

\bibitem[Mnih et~al.(2012)Mnih, Larochelle, and
  Hinton]{DBLP:journals/corr/abs-1202-3748}
V.~Mnih, H.~Larochelle, and G.~E. Hinton.
\newblock Conditional restricted {B}oltzmann machines for structured output
  prediction.
\newblock \emph{CoRR}, abs/1202.3748, 2012.

\bibitem[Mont{\'u}far and Ay(2011)]{Montufar2011}
G.~Mont{\'u}far and N.~Ay.
\newblock Refinements of universal approximation results for deep belief
  networks and restricted {B}oltzmann machines.
\newblock \emph{Neural Computation}, 23\penalty0 (5):\penalty0 1306--1319,
  2011.

\bibitem[Mont\'ufar and Morton(2015)]{montufar2012does}
G.~Mont\'ufar and J.~Morton.
\newblock When does a mixture of products contain a product of mixtures?
\newblock \emph{SIAM Journal on Discrete Mathematics}, 29:\penalty0 321--347,
  2015.
\newblock URL \url{http://dx.doi.org/10.1137/140957081}.

\bibitem[Mont\'ufar and Rauh(2014)]{montufar2014scaling}
G.~Mont\'ufar and J.~Rauh.
\newblock Scaling of model approximation errors and expected entropy distances.
\newblock \emph{Kybernetika}, 50\penalty0 (2):\penalty0 234--245, 2014.

\bibitem[Mont\'ufar et~al.(2011)Mont\'ufar, Rauh, and Ay]{NIPS2011_4380}
G.~Mont\'ufar, J.~Rauh, and N.~Ay.
\newblock Expressive power and approximation errors of restricted {B}oltzmann
  machines.
\newblock In J.~Shawe-Taylor, R.~Zemel, P.~Bartlett, F.~Pereira, and
  K.~Weinberger, editors, \emph{Advances in Neural Information Processing
  Systems 24}, pages 415--423. Curran Associates, Inc., 2011.
\newblock URL
  \url{http://papers.nips.cc/paper/4380-expressive-power-and-approximation-errors-of-restricted-boltzmann-machines.pdf}.

\bibitem[Mont\'ufar et~al.(2013)Mont\'ufar, Rauh, and Ay]{GSI2013}
G.~Mont\'ufar, J.~Rauh, and N.~Ay.
\newblock Maximal information divergence from statistical models defined by
  neural networks.
\newblock In F.~Nielsen and F.~Barbaresco, editors, \emph{Geometric Science of
  Information}, LNCS 8085, pages 759--766. Springer, 2013.
\newblock URL \url{http://dx.doi.org/10.1007/978-3-642-40020-9_85}.

\bibitem[Mont\'ufar et~al.(2014)Mont\'ufar, Ghazi-Zahedi, and
  Ay]{montufar2014framework}
G.~Mont\'ufar, K.~Ghazi-Zahedi, and N.~Ay.
\newblock A theory of cheap control in embodied systems.
\newblock \emph{arXiv preprint arXiv:1407.6836}, 2014.

\bibitem[Ojha(2000)]{857765}
P.~C. Ojha.
\newblock Enumeration of linear threshold functions from the lattice of
  hyperplane intersections.
\newblock \emph{Neural Networks, IEEE Transactions on}, 11\penalty0
  (4):\penalty0 839--850, Jul 2000.
\newblock ISSN 1045-9227.
\newblock \doi{10.1109/72.857765}.

\bibitem[Ross(1983)]{Ross:1983:ISD:538843}
S.~M. Ross.
\newblock \emph{Introduction to Stochastic Dynamic Programming: Probability and
  Mathematical}.
\newblock Academic Press, Inc., Orlando, FL, USA, 1983.

\bibitem[Salakhutdinov et~al.(2007)Salakhutdinov, Mnih, and
  Hinton]{Salakhutdinov:2007:RBM}
R.~Salakhutdinov, A.~Mnih, and G.~E. Hinton.
\newblock Restricted {B}oltzmann machines for collaborative filtering.
\newblock In \emph{Proceedings of the 24th International Conference on Machine
  Learning (ICML 2007)}, pages 791--798, New York, NY, USA, 2007. ACM.

\bibitem[Sallans and Hinton(2004)]{Sallans:2004:RLF:1005332.1016794}
B.~Sallans and G.~E. Hinton.
\newblock Reinforcement learning with factored states and actions.
\newblock \emph{Journal of Machine Learning Research}, 5:\penalty0 1063--1088,
  2004.

\bibitem[Smolensky(1986)]{Smolensky1986}
P.~Smolensky.
\newblock Parallel distributed processing: Explorations in the microstructure
  of cognition, vol.~1.
\newblock In D.~E. Rumelhart, J.~L. McClelland, and C.~PDP Research~Group,
  editors, \emph{Parallel Distributed Processing: Volume 1: Foundations},
  chapter Information Processing in Dynamical Systems: Foundations of Harmony
  Theory, pages 194--281. MIT Press, Cambridge, MA, USA, 1986.
\newblock URL \url{http://dl.acm.org/citation.cfm?id=104279.104290}.

\bibitem[Sutskever and Hinton(2007)]{sutskever_hinton_07}
I.~Sutskever and G.~E. Hinton.
\newblock Learning multilevel distributed representations for high-dimensional
  sequences.
\newblock In M.~Meila and X.~Shen, editors, \emph{AISTATS}, volume~2 of
  \emph{JMLR Proceedings}, pages 548--555. JMLR.org, 2007.

\bibitem[Taylor et~al.(2007)Taylor, Hinton, and Roweis]{Taylor06modelinghuman}
G.~W. Taylor, G.~E. Hinton, and S.~T. Roweis.
\newblock Modeling human motion using binary latent variables.
\newblock In B.~Sch\"{o}lkopf, J.~Platt, and T.~Hoffman, editors,
  \emph{Advances in Neural Information Processing Systems 19}, pages
  1345--1352. MIT Press, 2007.
\newblock URL
  \url{http://papers.nips.cc/paper/3078-modeling-human-motion-using-binary-latent-variables.pdf}.

\bibitem[van~der Maaten(2011)]{Maaten2011}
L.~van~der Maaten.
\newblock Discriminative restricted {B}oltzmann machines are universal
  approximators for discrete data.
\newblock Technical Report EWI-PRB TR 2011{001}, Delft University of
  Technology, 2011.

\bibitem[Varshamov(1957)]{Varshamov:1957}
R.~R. Varshamov.
\newblock Estimate of the number of signals in error correcting codes.
\newblock \emph{Doklady Akad. Nauk SSSR}, 117:\penalty0 739--741, 1957.

\bibitem[Wenzel et~al.(2000)Wenzel, Ay, and
  Pasemann]{Wenzel:2000:HAS:361159.361171}
W.~Wenzel, N.~Ay, and F.~Pasemann.
\newblock Hyperplane arrangements separating arbitrary vertex classes in
  n-cubes.
\newblock \emph{Adv. Appl. Math.}, 25\penalty0 (3):\penalty0 284--306, 2000.
\newblock URL \url{http://dx.doi.org/10.1006/aama.2000.0701}.

\bibitem[Younes(1996)]{Younes1996109}
L.~Younes.
\newblock Synchronous {B}oltzmann machines can be universal approximators.
\newblock \emph{Applied Mathematics Letters}, 9\penalty0 (3):\penalty0 109 --
  113, 1996.
\newblock URL
  \url{http://www.sciencedirect.com/science/article/pii/0893965996000419}.

\bibitem[Zeiler et~al.(2009)Zeiler, Taylor, Troje, and Hinton]{zeilerPigeon}
M.~Zeiler, G.~Taylor, N.~Troje, and G.~E. Hinton.
\newblock Modeling pigeon behaviour using a conditional restricted {B}oltzmann
  machine.
\newblock In \emph{17th European Symposium on Artificial Neural Networks
  (ESANN)}, 2009.

\end{thebibliography}
\bibliographystyle{abbrvnat}

\end{document}